\documentclass[10pt]{article}
\usepackage{latexsym}
\usepackage{graphicx}
\usepackage{amsmath, amssymb, amsfonts}
\usepackage{amsthm}
\usepackage{epsfig}
\usepackage{fancybox}
\usepackage{pstcol}
\usepackage{natbib,url}
\usepackage{multirow}
\usepackage{ulem}
\usepackage{hyperref}

\DeclareMathOperator*{\argmin}{\arg\!\min}

\def\cA{{\cal A}}

\def\cH{{\cal H}}
\def\cF{{\cal F}}

\def\cN{{\cal N}}

\def\cX{{\cal X}}
\def\cY{{\cal Y}}

\def\cC{{\cal C}}

\def\cE{{\cal E}}

\def\hat{\widehat}
\newcommand{\mbR}{\mathbb{R}}

\newcommand{\mbN}{\mathbb{N}}

\newcommand{\bX}{{\bf X}}
\newcommand{\bx}{{\bf x}}
\newcommand{\by}{{\bf y}}
\newcommand{\bW}{{\bf W}}

\newcommand{\bI}{{\bf I}}

\newcommand{\bfm}{{\bf m}}

\newcommand{\bfb}{{\bf b}}

\newcommand{\E}{\mbox{{\rm E}}}

\newcommand{\bc}{\begin{center}}
\newcommand{\ec}{\end{center}}
\newcommand{\be}{\begin{equation}}
\newcommand{\ee}{\end{equation}}
\newcommand{\ba}{\begin{array}}
\newcommand{\ea}{\end{array}}
\newcommand{\bean}{\begin{eqnarray*}}
\newcommand{\eean}{\end{eqnarray*}}
\newcommand{\bea}{\begin{eqnarray}}
\newcommand{\eea}{\end{eqnarray}}

\newtheorem{lemma}{\bf Lemma}
\newtheorem{theorem}{\bf Theorem}

\newtheorem{proposition}{\bf Proposition}

\newcommand{\ben}{\begin{enumerate}}
\newcommand{\een}{\end{enumerate}}
\newcommand{\bed}{\begin{itemize}}
\newcommand{\eed}{\end{itemize}}

\numberwithin{equation}{section}

\newenvironment{equations*}{\equation\aligned\nonumber}{\endaligned\endequation}

\begin{document}

\title{\bf Fast convergence rates of deep neural networks for classification}

\author{{\normalsize Yongdai Kim, Ilsang Ohn, and Dongha Kim}\\
{\normalsize Department of Statistics, Seoul National University, Seoul, Korea}\\
\\
}
\maketitle

\begin{abstract}
We derive the fast convergence rates of a deep neural network (DNN) classifier with the
rectified linear unit (ReLU) activation function learned using the hinge loss. We consider three cases for a true model:
(1) a smooth decision boundary, (2) smooth conditional class probability, and (3) the margin condition
(i.e., the probability of inputs near the decision boundary is small). We show that
the DNN classifier learned using the hinge loss achieves fast rate convergences for all three
cases provided that the architecture (i.e., the number of layers, number of nodes and sparsity).
is carefully selected. An important implication is that DNN architectures are very flexible 
for use in various cases without much modification. In addition, we consider a DNN
classifier learned by minimizing the cross-entropy, and show that
the DNN classifier achieves a fast convergence rate under the condition that
the conditional class probabilities of most data are sufficiently close to either 1 or zero.
This assumption is not unusual for image recognition because human beings are  extremely good at 
recognizing most images. To confirm our theoretical explanation, we present the results of a small numerical study conducted to compare the hinge loss and cross-entropy.

\noindent
 Keywords: Classification, Deep neural network, Excess risk, Fast convergence rate 
  \end{abstract}

\section{Introduction}

Deep learning \citep{hinton2006reducing,larochelle2007empirical,goodfellow2016deep} has received much attention for dimension reduction and classification of objects, such as images, speech, and language. Various supervised/unsupervised deep learning architectures, such as deep belief network \citep{hinton2006fast}, have been developed and applied to large scale real data with great success. A key ingredient for the success of deep learning is to discover multiple levels of representation of the given dataset with higher levels of representation defined hierarchically in terms of lower level representations. The central motivation is that higher-level representations can potentially capture relevant higher-level abstractions. See \citet{goodfellow2016deep} for details.

Theoretical explanations regarding the success of deep learning have been recently studied.
Many researchers have demonstrated that deep neural networks (DNNs) are much more efficient in representing certain complex functions than their shallow counterparts \citep{montufar2014number, raghu2016expressive, eldan2016power}, which has been reconfirmed by 
\cite{yarotsky2017error} and \cite{petersen2018optimal},
who showed that DNNs can approximate a large class of functions,
including even discontinuous functions with a parsimonious number of parameters.
In turn, using this efficient approximation property of a DNN, \cite{schmidt2017nonparametric} and \cite{imaizumi2018deep} proved that, for regression problems,  we can estimate a complex function including a discontinuous function
using a DNN with the (in the minimax sense) optimal convergence rate.
A surprising result is that any linear estimators, which include the ridge penalized kernel estimator, are sub-optimal in estimating a discontinuous function while the DNN is optimal.

In this paper, we consider classification problems. 
It is known that estimating the classifier directly instead of estimating the conditional class probability (i.e., $\eta(\bx)=\Pr(Y=1|\bX=\bx)$)
can help achieve fast convergence rates 
\citep{mammen1999smooth, tsybakov2004optimal, tsybakov2005square, audibert2007fast}
under the Tsybakov's low noise condition.
We prove that the estimation of a classifier based on the DNN with the hinge loss can achieve fast convergence rates under various situations.

In practice, estimating the classifier directly is difficult because the classifier itself is discontinuous. \cite{mammen1999smooth, tsybakov2004optimal, tsybakov2005square} considered estimating the classifier directly, which may be computationally infeasible in practice. Under the smoothness assumption on the conditional class probability,  \citet{audibert2007fast} estimated the conditional class probability using a local polynomial estimator and
obtained a plug-in classifier. Finding the best plug-in classifier, however, requires 
searching in a given sieve, which is computationally demanding. In contrast, learning a DNN is relatively straightforward owing to the gradient descent algorithm, despite a risk of arriving at bad local minima.
 
We consider three cases regarding a true classifier: (1) a smooth boundary, 
(2) smooth conditional class probability, and (3) the margin condition (i.e., the probability 
of the inputs near the decision boundary is small).
We prove that the DNN classifier can achieve fast convergence rates for
all of these three cases if the architecture (i.e., the number of layers, number of nodes, and sparsity of the weights) of the DNN is carefully selected.
In particular, the DNN classifier is minimax optimal
for a smooth conditional class probability, and achieves faster convergence rates under the
margin condition. To the best of the authors' knowledge, no other estimator achieves fast convergence rates for these three cases simultaneously. 
 
The cross-entropy is the standard objective function used in learning a DNN, and is an empirical risk with respect to the logistic loss (i.e., the negative log-likelihood of the logistic model). It is well known that the logistic loss
estimates the conditional class probability rather than the classifier, and hence
will be sub-optimal. However, learning a DNN with the cross-entropy
performs quite well in practice. We justify the use of the cross-entropy in learning a DNN
by showing that the corresponding classifier also achieves a fast convergence rate when
most data have a conditional class probability close to  1 or zero.
Note that this assumption is reasonable for image recognition because
human beings recognize most real world images quite well.

The remainder of this paper is organized as follows.
Section 2 describes the hinge loss and DNN classifier.
Section 3 derives the convergence rates of the excessive risk of a DNN classifier
for the aforementioned three cases regarding a true model. The fast convergence rate of the DNN classifier with the cross-entropy is derived in Section 4, and concluding remarks follow in Section 5.

\subsection{Notations}

For a function $f:\cX\to\mbR$, where $\cX$ denotes the domain of the function,  
let $\|f\|_\infty=\sup_{\bx\in\cX}|f(\bx)|$. For a given subset $B$ of $\cX$, we let $\|f\|_{\infty, B}=\sup_{\bx\in B}|f(\bx)|$.

For two given sequences $\{a_n\}_{n\in \mbN}$ and $\{b_n\}_{n\in \mbN}$ of real numbers, we write $a_n\lesssim b_n$ if there exists a constant $C>0$ such that $a_n\le C b_n$ 
for all sufficiently large $n$.
In addition, we write $a_b\asymp b_n$ if $a_n\lesssim b_n$ and $a_n\gtrsim b_n$.
For $N\in\mbN$, we let $[N]=\{1,\dots, N\}$.

Let $\bfm=(m_1,\dots, m_d)\in \mbN_0^d$ be a multiple index, where $\mbN_0=\mbN\cup\{0\}$. We define $|\bfm|=m_1+\cdots+m_d$ and $\bx^{\bfm}=x_1^{m_1}\cdots x_d^{m_d}$ for a multiple index $\bfm$.  For $f:\cX\to\mbR$ and $\bfm\in \mbN_0^d$, let
$$\partial^{\bfm}f=\frac{\partial^{|\bfm|}f}{\partial \bx^{\bfm}}=\frac{\partial^{|\bfm|}f}{\partial x_1^{m_1} \cdots\partial x_d^{m_d}},$$
and for $s\in(0,1]$, let 
$$[f]_{s,\cX}=\sup_{\bx,\by\in \cX, \bx\neq\by }\frac{|f(\bx)-f(\by)|}{|\bx-\by|^s}.$$
We denote by $\cC^m(\cX)$ and $m\in \mbN$, the space of $m$ times differentiable functions on $\cX$ whose partial derivatives of order $\bfm$ with $|\bfm|\le m$ are continuous. For a positive real value $\alpha$, we write $\alpha=[\alpha]^-+\{\alpha\}^+$, where $[\alpha]^-=\lceil \alpha-1 \rceil\in \mbN_0$ and $\{\alpha\}^+=\alpha-[\alpha]^-\in (0,1]$. The H\"older space of order $\alpha$ is defined as
$\cH^{\alpha}(\cX)=\left\{f\in \cC^{[\alpha]^-}(\cX):\|f\|_{\cH^{\alpha}(\cX)}< \infty\right\}$,
where $\|f\|_{\cH^{\alpha}(\cX)}$ denotes the H\"older norm defined by
$$\|f\|_{\cH^{\alpha}(\cX)}=\max_{|\bfm|\le [\alpha]^-}\|\partial^{\bfm}f\|_{\infty, \cX}+\max_{|\bfm|=[\alpha]^-}[\partial^{\bfm}f]_{\{\alpha\}^+, \cX}.$$
We let 
$$\cH^{\alpha, r}(\cX)=\left\{f\in \cC^{[\alpha]^-}(\cX):\|f\|_{\cH^{\alpha}(\cX)}\le r\right\},$$ 
which is a closed ball in the H\"older space of radius $r$ with respect to the H\"older norm.

\section{Estimation of the classifier with DNNs}

We consider a binary classification problem.
The data are given as $(\bx_1,y_1),\ldots,(\bx_n,y_n)$, where
$\bx_i\in \cX\subset \mbR^d$ are  input vectors, and $y_i\in \{-1,1\}$ are class labels.
Here, for simplicity, we set $\cX=[0,1]^d$; however, this can be extended to any compact subset of $\mbR^d$.
We assume that $(\bx_i,y_i)$ are independent copies of a random vector 
$(\bX,Y) \sim \Pr$ for a certain probability measure $\Pr$. We let $P_X$ be the marginal distribution of $\bX$ induced by the joint distribution $\Pr$.

\subsection{Necessity of the hinge loss}

Before going further, we will first review why we consider the hinge loss instead of the logistic loss to achieve fast convergence rates.
Let $\cC$ be the class of all classifiers (i.e., all measurable mapping from $\cX$ to $\{-1,1\})$.
The objective of classification is to find the optimal classifier (called the Bayes classifier) $C^*$, which is defined as
$$C^*=\operatorname*{argmin}_{C\in \cC} \E \left[ \mathbf{1}\{C(\bX)\ne Y\}\right],$$
where $\mathbf{1}\{\cdot\}$ is 1 if $\{\cdot\}$ is true, and is 0 otherwise.

Because we do not know the probability measure $\Pr$ generating data, we cannot find $C^*$. Instead, we estimate $C^*$ based on the training data. The most popular method for estimating $C^*$ is the empirical risk minimization approach, where we estimate $C^*$ by minimizing the empirical risk. That is, we estimate $C^*$ using $\hat{C}$, where
\be
\label{eq:emr_1}
\hat{C}=\operatorname*{argmin}_{C\in \cC_n} \sum_{i=1}^n \mathbf{1}\{C(\bx_i)\ne y_i\}/n,
\ee
where $\cC_n$ is a given class of classifiers depending on the sample size $n$.

In practice, $\hat{C}$ is not computationally feasible because minimizing the empirical risk
with the 0-1 loss over $\cC_n$ is NP hard \citep{bartlett2006convexity}. 
An alternative approach is to replace the 0-1 loss with other computationally easier losses so-called surrogate losses. In addition, instead of a class of classifiers $\cC_n$, 
we consider a class of real-valued functions $\cF_n$. For a given surrogate loss $\phi$,
we estimate $\hat{f}$ by minimizing the surrogate empirical risk (or empirical $\phi$-risk)
\be
\label{eq:emr_2}
\cE_{\phi,n}(f)=\sum_{i=1}^n \phi( y_i f(\bx_i))/n
\ee
on $\cF_n$, and construct a classifier by $\hat{C}(\bx)={\rm sign} \hat{f}(\bx)$.

A question in using a convex surrogate loss is the relation between the minimizer of the 0-1 empirical risk (\ref{eq:emr_1}) and that of the empirical $\phi$-risk (\ref{eq:emr_2}). Because  the empirical $\phi$-risk converges to
the population $\phi$-risk $\cE_\phi(f)=\E(\phi(Yf(\bX))$ for a given $f$ by the law
of large numbers, we can consider $\hat{f}$ as an estimator of $f^*_{\phi}$, which is defined as
$$f^*_{\phi}= \operatorname*{argmin}_{f\in\cF_\infty} \cE_\phi(f),$$
where $\cF_\infty$ is the limit of $\cF_n$ in a certain sense.
When $\cF_\infty$ is the set of all measurable functions, we say that
the surrogate loss $\phi$ is Fisher consistent if ${\rm sign} (f_\phi^*(\bx))=C^*(\bx)$.

It is known \citep{lin2004note, bartlett2006convexity} that the Fisher consistency holds under very mild conditions on $\phi$. 
In particular,  $f_\phi^*$ is known for various surrogate losses.
For example, when $\phi$ is the logistic loss 
(i.e., $\phi(z)=\log (1+\exp(-z))$), we have
$f_\phi^*(\bx)=\log \eta(\bx)/(1-\eta(\bx))$,
where $\eta(\bx)=\Pr(Y=1|\bX=\bx)$ \citep{friedman2000additive}. Hence, the logistic loss satisfies the Fisher consistency, which justifies the use of the cross-entropy when learning a deep neural network. That is, deep learning with the cross-entropy essentially estimates the log odds of the conditional class probability. 

As we explained in the Introduction, it would be better to estimate the Bayes classifier directly, which is realized conceptually if $f^*_\phi$ is the Bayes classifier.
The hinge loss $\phi(z)=(1-z)_+=\max\{1-z, 0\}$ has such a property \citep{lin2002support}, which is why we consider the hinge loss.
Note that there are other losses that have $f^*_\phi=C^*$.
An example is the $\psi$-loss \citep{shen2003psi}, which is also known as the ramp loss \citep{collobert2006large}. Although the $\psi$-loss has many advantages over the hinge loss, the $\psi$-loss is nonconvex, and learning
a DNN classifier using the $\psi$-loss would be extremely difficult
because the DNN classifier is nonconvex as well. 

\subsection{Learning DNN with the hinge loss}

We consider DNNs that take $d$-dimensional inputs and produce one-dimensional outputs. A DNN with $L$ many layers, and $\{N^{(l)}, l\in[L]\}$ many nodes at each layer, is defined as
$$z_j^{(l)}(\bx)=b_j^{(l)}+\sum_{k=1}^{N^{(l-1)}} W_{j,k}^{(l)} h_k^{(l-1)}(\bx)$$
and $$h_{j}^{(l)}(\bx)=\sigma(z_j^{(l)}(\bx))$$
for $l=1,\ldots,L$ and
$$f(\bx)=b^{(L+1)}+ \sum_{k=1}^{N^{(L)}} W_{1,k}^{(L+1)} h_k^{(L)}(\bx)$$
with $N^{(0)}=d$ and $h_k^{(0)}(\bx)=x_k$.
We consider the ReLU activation function $\sigma(z)=(z)_+$.
We denote $f(\bx)$ as $f(\bx|\Theta)$, where $\Theta=((\bW^{(l)}, \bfb^{(l)}))_{l=1,\dots, L+1}$ is the parameter set including all weights and biases. 

For the given $\Theta$, let $|\Theta|$ be the number of layers in $\Theta$. Let $N_{\max}(\Theta)$ be the maximum number of nodes, that is, $f(\cdot|\Theta)$ has at most $N_{\max}(\Theta)$ nodes at each layer. We define $\|\Theta\|_0$ as the number of nonzero parameters in $\Theta$,
$$\|\Theta\|_0=\sum_{l=1}^{L+1}\left( \|\text{vec}(\bW^{(l)})\|_0 +\|\bfb^{(l)}\|_0\right),$$
where $\text{vec}(\bW^{(l)})$ transforms the matrix $\bW^{(l)}$ into the corresponding vector by
concatenating the column vectors.
Similarly, we define $\|\Theta\|_\infty$ as the largest absolute value of the parameters in $\Theta$,
$$\|\Theta\|_\infty = \max \left\{ \max_{1\le l\le L+1} \|\text{vec}(\bW^{(l)})\|_\infty, 
\max_{1\le l\le L+1} \|\bfb^{(l)}\|_\infty\right\}.$$
For a given $n$, let $\cF_n$ be
    \begin{align*}
    \cF_n &= \cF^{\textup{DNN}}(L_n, N_n, S_n, B_n, F_n)\\
    &=\big\{f(\bx|\Theta):  |\Theta|\le L_n, N_{\max}(\Theta)\le N_n, \|\Theta\|_0\le S_n, \\
    	&\qquad\qquad\qquad \|\Theta\|_\infty\le B_n, \|f(\cdot|\Theta)\|_\infty \le F_n\big\}
    \end{align*}
where the positive constants $L_n$, $N_n$, $S_n$, $B_n$, and $F_n$ are specified later.

We let $\hat{f}^{\textup{DNN}}_{\phi, n}$ be the minimizer of $\cE_{\phi,n}(f)$ over $\cF_n$ for a given surrogate loss $\phi$, i.e.,
    \begin{equation}
    \label{eq:fdnn}
    \hat{f}^{\textup{DNN}}_{\phi, n}=\argmin_{f\in \cF_n}\frac{1}{n}\sum_{i=1}^n\phi(y_if(\bx_i)).
    \end{equation}
In the following section, we prove
the fast convergence rates of $\hat{f}^{\textup{DNN}}_{\phi, n}$ for various 
cases of the true model when $\phi$ is the hinge loss and
$L_n$, $N_n$ $S_n$, $B_n$, and $F_n$ are carefully selected.
For detailed formulas of
$L_n,N_n,S_n, B_n$, and $F_n$ in terms of the sample size $n$, see the proofs of the corresponding theorems in the Appendix.

\section{Fast convergence rates of DNN classifiers with the hinge loss}

In this section, we consider the hinge loss and derive the convergence rates of the excess risk of $\hat{f}^{\textup{DNN}}_{\phi, n}$. For a given function $f$, the \textit{excess risk} of $f$  is defined as
$$\cE(f, C^*)=\cE(f)-\cE(C^*)=\E[\mathbf{1}(Yf(\bX)<0)]-\E[\mathbf{1}(YC^*(\bX)<0)],$$
and the \textit{excess $\phi$-risk}  of $f$   is defined by
$$\cE_\phi(f, f^*_\phi)=\cE_\phi(f)-\cE_\phi(f^*_\phi)=\E[\phi(Yf(\bX))]-\E[\phi(Yf^*_\phi(\bX))].$$

Throughout this paper, we always assume the Tsybakov noise condition (\cite{mammen1999smooth,tsybakov2004optimal}). 
\begin{itemize}
    \item[(N)] There exists $C>0$ and $q\in[0,\infty]$ such that for any  $t>0$
    $$\Pr\left(\{\bX:|2\eta(\bX)-1|\le t\}\right)\le C t^q.$$
\end{itemize}
We call the parameter $q$ appearing in assumption (N) the \textit{noise exponent}. 

We consider three cases regarding a true model: (1) a smooth decision boundary,
(2) smooth class conditional probability, and (3) the margin condition. 
We derive the fast convergence rates of the DNN classifier using the hinge loss for all 
three cases.

\subsection{Case 1: Smooth boundary}

To describe the smooth Bayes classifier, we introduce the notion of piecewise constant functions with smooth boundaries. We adopt the notations and definitions from \cite{petersen2018optimal} and \cite{imaizumi2018deep}.
For $g\in \cH^{\alpha, r}([0,1]^{d-1})$ and $j\in [d]$, we define a \textit{horizon function} $\Psi_{g,j}:[0,1]^d\to \{0,1\}$ as
$$
	\Psi_{g,j}(\bx)= \mathbf{1}(x_j\ge g(\bx_{-j})),  
$$
where  $\bx_{-j}=(x_1,\dots, x_{j-1}, x_{j+1}, \dots, x_d)$. For each horizon function, we define the corresponding \textit{basis piece} $I_{g,j}$ as
$$I_{g,j}=\left\{\bx\in[0,1]^d:\Psi_{g,j}(\bx)=1\right\}.$$
We define a \textit{piece} by the intersection of $K$ basis pieces. The set of pieces is denoted by
$$\cA^{\alpha, r, K}=\left\{A\subset[0,1]^d:A=\bigcap_{k=1}^K I_{g_k, j_k}, g_k\in \cH^{\alpha, r}([0,1]^{d-1}), j_k\in[d]\right\}.$$
Let $\cC^{\alpha,r,K, T}$ be the set of classifiers of the form
$$C(\bx)=2 \sum_{t=1}^T \mathbf{1}(\bx\in A_t)-1,$$
for $T\in \mathbb{N}$, and disjoint subsets $A_1,\ldots,A_T$ of $\cX$ in $\cA^{\alpha,r,K}$.
In this subsection, we assume that the Bayes classifier belongs to $\cC^{\alpha,r,K,T}$.

The following theorem proves the convergence rate of the DNN classifier
with the hinge loss.

\begin{theorem}
\label{th:main_sb}
Assume \textup{(N)} using the noise exponent  $q\in[0,\infty]$. 
If the surrogate loss $\phi$ is the hinge loss, the classifier $\hat{f}_{\phi, n}^{\textup{DNN}}$ defined by (\ref{eq:fdnn}) with carefully selected $L_n,N_n,S_n, B_n$, and $F_n$
satisfies
\be
\label{eq:rate1}
\sup_{C^*\in C^{\alpha,r,K,T}}
\E\left[\cE(\hat{f}^{\textup{DNN}}_{\phi, n}, C^*)\right]\lesssim  
\left(\frac{\log^3 n}{n}\right)^{\frac{\alpha (q+1)}{\alpha (q+2)+(d-1)(q+1)}}
\ee
where the expectation is taken over the training data.
\end{theorem}

\cite{tsybakov2004optimal} showed that the minimax lower bound is given by
$$\inf_{\hat{f}_n}\sup_{C^*\in C^{\alpha,r,1,1}}
 \E\left[\cE(\hat{f}_n, C^*)\right]\gtrsim n^{-\frac{\alpha (q+1)}{\alpha (q+2)+(d-1)q}},$$
where the infimum is taken over all classifiers $\hat{f}_n:(\cX\times\cY)^n\mapsto \cF$, where $\cF$ is a set of all measurable functions. 
Unfortunately, the convergence rate (\ref{eq:rate1}) is not optimal in the minimax sense. 
However, the difference becomes small when the noise exponent $q$ is large. 
Note that the estimators in \cite{mammen1999smooth} and \cite{tsybakov2004optimal}
have slower convergence rates than that in (\ref{eq:rate1}) when $\alpha< d-1$.
However, the estimator of \cite{tsybakov2005square} achieves the minimax lower bound 
for any $\alpha>0$. At this point, we do not know whether the sub-optimal 
convergence rate (\ref{eq:rate1}) is inevitable owing the use of the hinge loss rather than the 0-1 loss. We will pursue this issue in the near future.

\subsection{Case 2: Smooth conditional class probability}

We assume that $\eta(\bx)$ is smooth. That is, 
$\eta(\cdot)\in \cH^{\beta,r}([0,1]^d)$ for some $\beta>0$ and $r>0$. The following 
theorem provides the convergence rate of the DNN classifier.

\begin{theorem}
\label{th:main_scp}
Assume \textup{(N)} with the noise exponent $q\in[0,\infty]$. If the surrogate loss $\phi$ is the hinge loss, the classifier $\hat{f}_{\phi, n}^{\textup{DNN}}$ defined by (\ref{eq:fdnn}) with carefully selected $L_n,N_n,S_n, B_n$, and $F_n$ satisfies
\be
\label{eq:rate_scp}
\sup_{\eta\in \cH^{\beta, r}}
\E\left[\cE(\hat{f}^{\textup{DNN}}_{\phi, n}, C^*)\right]\lesssim  
\left(\frac{\log^3 n}{n}\right)^{\frac{\beta (q+1)}{\beta (q+2)+d}}.
\ee
\end{theorem}

 \cite{audibert2007fast} showed that when $\eta(\cdot)\in \cH^{\beta}([0,1]^d)$, the minimax lower bound of the excess risk is given by
 $$\inf_{\hat{f}_n}\sup_{\eta\in \cH^{\beta, r}}
 \E\left[\cE(\hat{f}_n, C^*)\right]\gtrsim n^{-\frac{\beta(q+1)}{\beta (q+2)+d}}.$$
Hence, the convergence rate (\ref{eq:rate_scp}) is minimax optimal up to a logarithmic factor.

\subsection{Case 3: Margin condition}

The convergence rate can be improved if we assume that the density of an input vector
is small around the decision boundary. Let $B_\epsilon^*=\{\bx: \text{dist}(\bx,D^*)\le \epsilon\}$,
where $D^*=\{\bx:\eta(x)=1/2\}$ and $\text{dist}(\bx,D^*)=\inf_{\bx'\in D^*} \|\bx-\bx'\|_2$, where $\|\cdot\|_2$ denotes the Euclidian norm.
We introduce the following condition on the probability measure $P_X$.

\begin{itemize}
    \item[(M)] There exist $C>0$, $\epsilon_0>0$, and $\gamma\in[1,\infty]$ such that for any $\epsilon\in(0, \epsilon_0]$, 
    $$\Pr\left(\{\bX: \text{dist}(\bX,D^*)\le \epsilon\} \right)\le C \epsilon^\gamma.$$
\end{itemize}

The condition (M) is considered by \cite{steinwart2008support}, where
the parameter $\gamma$ in (M) is called the \textit{margin exponent}. 
\cite{steinwart2008support} proves that the support vector machine with the Gaussian kernel achieves a fast convergence rate under the condition (M). The following theorem proves that a similar convergence rate can be achieved using the DNN classifier. 

\begin{theorem}
\label{th:main_mar}
Assume \textup{(N)} with the noise exponent $q\in[0,\infty]$, and \textup{(M)} with the margin exponent $\gamma\in[1,\infty]$. If the surrogate loss $\phi$ is the hinge loss, the classifier $\hat{f}_{\phi, n}^{\textup{DNN}}$ defined by (\ref{eq:fdnn}) with carefully selected $L_n,N_n,S_n, B_n$, and $F_n$
satisfies
\be
\label{eq:rate1_2}
\sup_{C^{\star}\in C^{\alpha,r,K,T}}
\E\left[\cE(\hat{f}^{\textup{DNN}}_{\phi, n}, C^*)\right]\lesssim  
\left(\frac{\log^3 n}{n}\right)^{\frac{\alpha (q+1)}{\alpha (q+2)+(d-1)(q+1)/\gamma}}. 
\ee
\end{theorem}

An interesting feature of the convergence rate (\ref{eq:rate1_2}) is that the dependency of
the input dimension $d$ diminishes as $\gamma$ increases. In the extreme case where $\gamma\rightarrow \infty$, the convergence rate becomes $n^{-(q+1)/(q+2)}$
up to the logarithm factor, which depends on neither the smoothness of the boundary nor the dimension of the input. This partly explains why the DNN classifier works well with high-dimensional inputs such as images.

To investigate the validity of the margin condition (M), we explore the area near the decision boundary obtained by the cat and dog images of the CIFAR10 dataset. We first fit the decision boundary using a convolutional neural network (CNN) with cat and dog images in the CIFAR10 dataset. We then randomly select two images, one from dog and the other from cat, and take convex combinations of them to obtain a sequence of images between the two selected images. Figure \ref{fig:interpol} shows five sequences of images from five randomly selected pairs of dog and cat images.
The images in the red box, which are the interpolated images with weights of the dog images 
ranging from $0.3$ to $0.7$, are visually unrealistic, which suggests
that the image classification has a large margin exponent.

\begin{figure}
    \centering
    \hspace{-6mm}\includegraphics[scale=0.25]{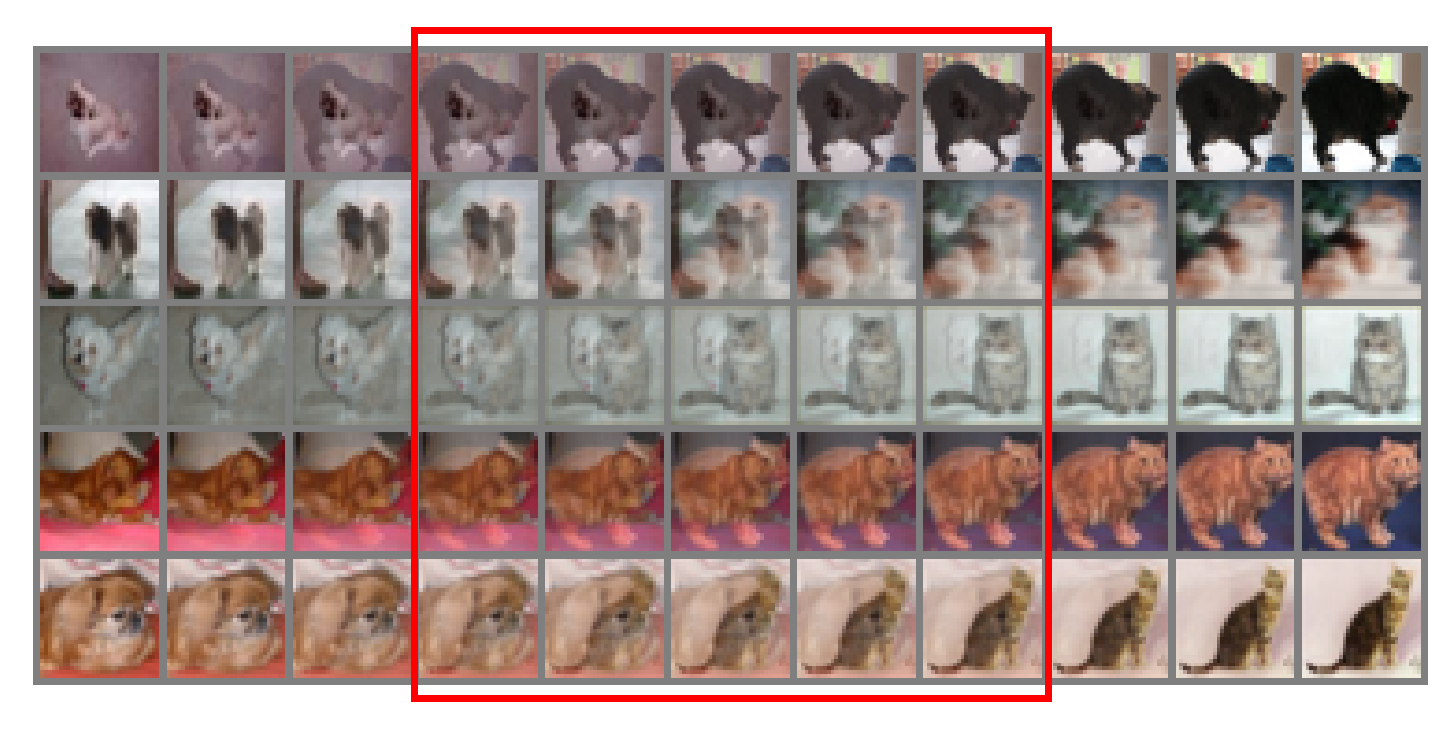}
    \caption{Interpolated images between cat and dog. The images in the red box seem to be unrealistic examples.}
    \label{fig:interpol}
\end{figure}

\subsection{Remarks regarding adpative estimation}

In practice, we know neither $q, \alpha,\beta$ nor $\gamma$, that affect the choice
of the DNN architecture parameters $L_n,N_n,S_n, B_n$, and $F_n$.
We may select them data-adaptively. General tool kits used to find an adaptive classifier have been developed by \cite{tsybakov2004optimal} and \cite{audibert2007fast}.
These tools can be applied to a DNN classifier with minor modification.

For example, the model selection approach with a data-split proposed by
\cite{audibert2007fast} can be applied without much hamper.
We first split the training data into two parts, $D_1$ and $D_2$, with the sample sizes
$n_1$ and $n_2$.
We then choose various values of $q,\alpha,\beta$, and $\gamma$, select the corresponding DNN architectures, and learn the architectures on data $D_1$. Finally,
among the learned DNN architectures, we choose the best DNN architecture
based on the data $D_2$. Because there is an algorithm of model selection where
the difference between the selected model and true model is $O_p(1/n)$
(for example, see \cite{juditsky2008learning} and \cite{audibert2007fast}), the selected model achieves the best possible convergence rate
$r_n^*$ as long as $n_1/n\rightarrow 1$ and $r_n^*/n_2 \rightarrow 0$.
We plan to report the detailed results of this soon.

\section{Use of cross-entropy}

The logistic loss does not estimate the classifier directly, and hence the convergence rate is sup-optimal in general. However, in practice, a DNN with the logistic loss (i.e., learned by minimizing the cross-entropy) works quite well.
In this section, we investigate when the logistic loss works well with a DNN.
We prove that the convergence rate of the excess risk of the DNN estimator with
the logistic loss can be fast when the true conditional class probabilities of most of  data are close to 1 or 0. This condition is expected to hold in most image recognition problems because human beings, who are thought to be a Bayes classifier, are very good at recognizing most images.
The formal statement of this condition is given as follows:
    \begin{itemize}
        \item[(E)]  For a given positive sequence $\{\tilde{F}_n\}_{n\in\mbN}$ with $\tilde{F}_n \to\infty$, there exists a positive sequence $\{\lambda_n\}_{n\in\mbN}$ with $\lambda_n\downarrow0$ such that
        $$\Pr\left\{\bX:|f_{\phi}^*(\bX)|> \tilde{F}_n\right\}\ge 1-\lambda_n.$$
    \end{itemize}


\begin{theorem}
\label{th:main_cr}
Assume \textup{(M)} with the margin exponent $\gamma\in[1,\infty]$. Let $\kappa=\alpha/(\alpha +(d-1)/\gamma)$. 
Assume \textup{(E)} with $\tilde{F}_n\asymp \kappa (\log n -3 \log (\log n))$ and $\lambda_n\asymp e^{-\tilde{F}_n}$. If $\phi$ is the logistic loss, then the classifier $\hat{f}_{\phi, n}^{\textup{DNN}}$ defined by (\ref{eq:fdnn}) with  $F_n=\tilde{F}_n$ and carefully selected $L_n,N_n,S_n$, and $B_n$ satisfies
\be
\label{eq:ratece}
\sup_{C^{\star}\in C^{\alpha,r,K,T}}
\E\left[\cE(\hat{f}^{\textup{DNN}}_{\phi, n}, C^*)\right]\lesssim  
 n^{-\kappa}\left(\log n\right)^{3\kappa+1}.
\ee
\end{theorem}

The convergence rate in Theorem \ref{th:main_cr} is equivalent to that in Theorem \ref{th:main_mar} for $q=\infty$ up to a logarithmic factor.

\begin{figure}
    \centering
    \includegraphics[scale=.9]{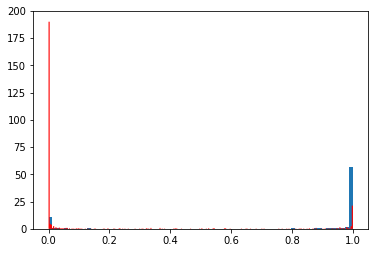}
    \caption{Histogram of the  conditional class probabilities estimated using a DNN with the logistic loss for CIFAR10 data. The blue bins are for the `dog' samples, and the red bins indicate  the `cat' samples.}
    \label{fig:hist}
\end{figure}

To investigate the validity of the condition (E), Figure \ref{fig:hist}
shows a histogram of the estimated conditional class probabilities of the test
data of the CIFAR10 data using the DNN classifier with the logistic loss. Note that most of
the conditional class probabilities are very close to either 1 or 0.

We compare the performance of the two DNN classifiers learned using the two surrogate losses - the logistic loss and the hinge loss. We analyze three benchmark datasets for image recognition, that is, MNIST, SVHN, and CIFAR10, where for each dataset we select two classes that are most difficult to recognize. The data descriptions and selected classes are summarized in Table \ref{tab:data}.
The detailed DNN architectures for the three datasets are given in Appendix \ref{sec:cnn}. The Adam is used for optimization with the learning rate $10^{-3}$.
Table \ref{tab:test} summarizes the test data error rates for various sizes of training data. The results are the averages (and standard errors) of 100 randomly selected training data, which amply show that the two estimators compete well with each other.

\begin{table}[]
\caption{Data summary}
    \centering
\begin{tabular}{lcccc}\hline
Data  & \# of training data & \# of test data & Input dimension & Selected classes \\\hline
MNIST & 60,000                  & 10,000              & $28\times28$           & `5' vs. `7'      \\
SVHN  & 73,257                  & 26,032              & $3\times32\times32$        & `4' vs. `9'      \\
CIFAR10 & 60,000                  & 50,000              & $3\times32\times32  $       & `cat' vs. `dog' \\ \hline
\end{tabular}
\label{tab:data}
\end{table}

\begin{table}[]
\caption{Test errors of the DNN classifiers learned using the hinge and logistic losses
with various training data sizes.}
\centering
\begin{tabular}{llcccc}\hline
\multirow{2}{*}{Data}    & \# of training & \multicolumn{2}{c}{Hinge loss} & \multicolumn{2}{c}{Logistic loss} \\
                         &  samples per each class                                         & Mean           & SE            & Mean            & SE              \\\hline
\multirow{3}{*}{MNIST}   
                         & 50                                      & 0.9318         & 0.0078        & 0.9359          & 0.0100          \\
                         & 500                                     & 0.9806         & 0.0031        & 0.9799          & 0.0024          \\
                         & 5000                                    & 0.9929         & 0.0006        & 0.9925          & 0.0005          \\\hline
\multirow{3}{*}{SVHN}    
                         & 50                                      & 0.7877         & 0.0698        & 0.7851          & 0.0798          \\
                         & 500                                     & 0.9500         & 0.0061        & 0.9545          & 0.0063         \\
                         & 5000                                    & 0.9796         & 0.0011        & 0.9801          & 0.0014          \\\hline
\multirow{3}{*}{CIFAR10} 
                         & 50                                      & 0.6628         & 0.0123        & 0.6698          & 0.0096          \\
                         & 500                                     & 0.7758         & 0.0090        & 0.7804          & 0.0081          \\
                         & 5000                                    & 0.8760         & 0.0064        & 0.8788          & 0.0047       
                         \\\hline
\end{tabular}
\label{tab:test}
\end{table}

\section{Concluding Remarks}

We showed that a DNN is very flexible in the sense that it achieves fast convergence rates
for various cases regarding a true model. It is interesting to note that a DNN is not only
good at estimating a smooth decision boundary but also a smooth conditional class probability.
In addition, a DNN can fully utilize the margin condition.  

We showed that using the cross-entropy is also
promising when the true conditional class probability is close to either 0 or 1 for most
data. However, we conjecture that learning a DNN by minimizing the cross-entropy would be sub-optimal when the conditional class probability is not extreme. 

Our theoretical results could be used to develop model selection procedures, particularly for the optimal selection of
$L_n$ and $N_n$. Moreover, it will be interesting to develop an online learning algorithm that
can select $L_n$ and $N_n$ data adaptively.

We did not consider a computational issue in this paper. Learning a DNN with a sparsity constraint has not been fully studied, although some methods have been proposed (e.g., \cite{liu2015sparse}, \cite{han2015learning}, and \cite{wen2016learning}).
A learning algorithm that supports our theoretical results will be worth pursuing.

\section*{Acknowledgement}

This work was supported by the Samsung Science and Technology Foundation under Project Number SSTF-BA1601-02.

\appendix
\section{Appendix}    

\subsection{Complexity measures of a class of functions}
We introduce the complexity measures of a given class of functions. 
Let $ \|\cdot\|_p$ for $1\le p<\infty$ be defined as $\|f\|_p=\left(\int_\cX|f(\bx)|^p\text{d}\mu(\bx)\right)^{1/p}$, where $\mu$ denotes the Lebesgue measure and $\|f\|_\infty=\sup_{\bx\in\cX}|f(\bx)|$.

Let $\cF$ be a given class of real-value functions defined on $\cX$. 
Let $\delta>0$ and $p\in[1,\infty]$. A collection $\{f_i\in\cF:i\in[N]\}$ is called a \textit{$\delta$-covering set} of $\cF$ with respect to the  $L_p$ norm if, for all $f\in\cF$, there exists $f_i$ in the collection such that $\|f-f_i\|_p\le \delta$. The cardinality of the minimal $\delta$-covering set is called the \textit{$\delta$-covering number} of $\cF$ with respect to the  $L_p$ norm, and is denoted by $\cN(\delta, \cF, \|\cdot\|_p)$, that is, 
$$\cN(\delta, \cF, \|\cdot\|_p)=\inf\left\{N\in\mbN:\exists f_1, \dots, f_N \mbox{ such that }
\cF\subset\bigcup_{i=1}^N B_p(f_i, \delta)\right\},$$
where $B_p(f_i, \delta)=\{f\in\cF:\|f-f_i\|_p\le \delta\}$.

A collection of pairs $\{(f_i^L, f_i^U)\in \cF\times\cF:i\in[N]\}$ is called a \textit{$\delta$-bracketing set} of $\cF$ with respect to the  $L_p$ norm if $\|f_i^U-f_i^L\|\le \delta$ for all $i\in [N]$, and for any $f\in \cF$, there is a pair $(f_i^L, f_i^U)$ in the collection such that $f_i^L\le f\le f_i^U$. The cardinality of the minimal $\delta$-bracketing set is called the \textit{$\delta$-bracketing number} of $\cF$ with respect to the $L_p$ norm, and is denoted by $\cN_B(\delta, \cF, \|\cdot\|_p)$. The \textit{$\delta$-bracketing entropy}, denoted by $H_B(\delta, \cF, \|\cdot\|_p)$ is the logarithm of the $\delta$-bracketing number, i.e., $H_B(\delta, \cF, \|\cdot\|_p)=\log \cN_B(\delta, \cF, \|\cdot\|_p)$.
 
For any $\delta>0$, it is known (see, for example, Lemma 2.1 of \cite{van2000empirical}) that 
$$\log\cN(\delta, \cF, \|\cdot\|_p)\le H_B(\delta, \cF, \|\cdot\|_p),$$
for any $p\in[1,\infty)$, and 
\be
\label{eq:entropy}
 H_B(\delta, \cF, \|\cdot\|_p)\le \log \cN(\delta/2, \cF, \|\cdot\|_\infty)
 \ee
if $\mu(\cX)=1$.

\subsection{Convergence rate of the excess $\phi$-risk for general surrogate losses}
\label{sec:gen}

In this subsection, we derive the convergence rate of the excess $\phi$-risk
under regularity conditions, which is used repeatedly in the following subsections.
The regularity conditions and techniques of the proof are minor modifications of those in \cite{park2009convergence}; however, we present the complete conditions and proof for the sake of readers' convenience.

We assume the following regularity conditions. 
 \begin{itemize}
    \item[(A1)] $\phi$ is Lipschitz, i.e., there exists a constant $C_1>0$ such that 
    $|\phi(z_1)-\phi(z_2)|\le C_1|z_1-z_2|$ for any $z_1, z_2\in \mbR$. 
    \item[(A2)] For a positive sequence $a_n=O(n^{-a_0})$ as $n\to\infty$ for some $a_0>0$, there exists a sequence of function classes $\{\cF_n\}_{n\in\mbN}$ such that
    $$\cE_\phi(f_n,f^*_{\phi})\le a_n$$
    for some  $f_n\in \cF_n$.
    \item[(A3)] 
    There exists a sequence $\{F_n\}_{n\in \mbN}$ with $F_n\gtrsim1$ such that $\sup_{f\in\cF_n}\|f\|_\infty\le F_n$.
    \item[(A4)] There exists a constant $\nu\in (0,1]$ such that for any $f\in \cF_n$ and any $n\in \mbN$,
   $$\E\left[\left\{\phi(Yf(\bX))-\phi(Yf^*_{\phi}(\bX))\right\}^2\right]
        \le C_2F_n^{2-\nu} \{\cE_\phi(f,f^*_{\phi})\}^\nu$$
    for a constant $C_2>0$ depending only on $\phi$ and $\eta(\cdot)$.
    \item[(A5)] For a positive  constant $C_3>0$, there exists a sequence  $\{\delta_n\}_{n\in\mbN}$ such that 
    $$H_B(\delta_n, \cF_n, \|\cdot\|_2)\le C_3n\left(\frac{\delta_n}{F_n}\right)^{2-\nu},$$
    for $\{\cF_n\}_{n\in\mbN}$ in (A2), $\{F_n\}_{n\in\mbN}$ in (A3), and $\nu$ in (A4).
\end{itemize}

For a proof of the general convergence result, we apply the large deviation inequality of \cite{shen1994convergence} presented in Lemma \ref{lem:empineq}. 

\begin{lemma}[Theorem 3 of \cite{shen1994convergence}]
\label{lem:empineq}
Let $\cF$ be the class of functions bounded above by $F$. Assume that $\E f(Z)=0$ for any $f\in\cF$ and $v\ge \sup_{f\in \cF}\textup{Var}(f(Z))$ for some $v>0$. 
Suppose that there exists $\zeta>0$ such that 
	\begin{enumerate}
	\item[(C1)] $ H_B(v^{1/2}, \cF, \|\cdot\|_2)\le \zeta nM^2/(8(4v+MF/3))$,
	\item[(C2)] $M\le  \zeta v/(4F)$, $v^{1/2}\le F$,
	\item[(C3)] if $\zeta M/8<v^{1/2}$, 
	$$M^{-1}\int_{\zeta M/32}^{v^{1/2}}H_B(u, \cF, \|\cdot\|_2)^{1/2}\textup{d} u\le \frac{n^{1/2} \zeta^{3/2}}{2^{10}}.$$
	\end{enumerate}
Then,
$${\Pr}^*\left(\sup_{f\in \cF}\frac{1}{n}\sum_{i=1}^n(f(Z_i)-\E f(Z_i))\ge  M\right)\le 3\exp\left\{-(1-\zeta)\frac{nM^2}{2(4v+MF/3)}\right\}$$
where ${\Pr}^*$ denotes the outer probability measure.
\end{lemma}

The following Theorem is the main result of this section, which gives the convergence
rate of the excess $\phi$-risk.

\begin{theorem}
\label{thm:gen}
Suppose that the conditions \textup{(A1)-(A5)} are met. Let $\epsilon_n^2\asymp \max\{a_n, \delta_n\}$. 
Then, the empirical $\phi$-risk minimizer $\hat{f}_{\phi,n}$ over $\cF_n$ satisfies 
$$\Pr\left(\cE_\phi(\hat{f}_{\phi,n}, f_{\phi}^*)\ge\epsilon_n^2\right)\lesssim \exp(-Cn(\epsilon_n^2/F_n)^{2-\nu}),$$
for some universal constant $C>0$.
\end{theorem}

\begin{proof}
Let $C_1$, $C_2$, and $C_3$ be constants appearing in assumptions (A1), (A4), and (A5), respectively. Let $\epsilon_n^2=\max\{ 2a_n, 2^7\delta_n/C_1\}$. We define the following empirical process
$$\E_n(f)=\frac{1}{n}\sum_{i=1}^n\left[\phi(y_if_n(\bx_i))-\phi(y_if(\bx_i))
-\E\left\{\phi(Yf_n(\bX))-\phi(Yf(\bX))\right\}\right]$$
where $f_n\in \cF_n$ is a function such that $\cE_\phi(f_n,f^*_{\phi})\le a_n$. 

Since $\hat{f}_n$ minimizes $\cE_{\phi, n}(f)=\frac{1}{n}\sum_{i=1}^n\phi(y_if(\bx_i))$, it follows that
    \begin{eqnarray*}
    \Pr\left\{\cE_\phi(\hat{f}_n, f^*_{\phi})\ge \epsilon_n^{2}\right\}
    \le{\Pr}^*\left(\sup_{f\in\cF_n:\cE_\phi(f, f^*_{\phi})\ge \epsilon_n^{2}}\frac{1}{n}\sum_{i=1}^n\{\phi(y_if_n(\bx_i))-\phi(y_if(\bx_i))\}\ge 0\right). 
    \end{eqnarray*}
We define
$$\cF_{n,i}=\{f\in\cF_n:2^{i-1}\epsilon_n^2\le \cE_\phi(f, f^*_{\phi})< 2^{i}\epsilon_n^2\}.$$
Note that for $i\in\mbN$ such that $2^{i-1}\epsilon_n^2>2C_1F_n$, $\cF_{n,i}$ is an empty set. This is because for any $f\in\cF_n$, $\|f\|_\infty\le F_n$, and thus
$\cE_\phi(f, f^*_{\phi})\le \E|\phi(Yf(\bX))-\phi(Yf^*_{\phi}(\bX))|\le C_1\E|f(\bX)-f^*_{\phi}(\bX)|\le 2C_1F_n$. Therefore, $\{f\in\cF_n:\cE_\phi(f, f^*_{\phi})\ge \epsilon_n^{2}\}\subset\bigcup_{i=1}^{i^*_n}\cF_{n,i}$, where $i^*_n=\inf\{i\in\mbN:2^{i-1}\epsilon_n^2> 2C_1F_n\}$. Thus, we only deal with $\cF_{n,i}$ using $i\le i^*_n$. Because $\cE_\phi(f_n,f^*_{\phi})\le a_n\le \epsilon_n^2/2$, we have 
    \begin{eqnarray*}
   \inf_{f\in \cF_{n,i}}\E\{\phi(Yf(\bX))-\phi(Yf_n(\bX))\}
   =\inf_{f\in \cF_{n,i}}\{\cE_\phi(f,f^*_{\phi})-\cE_\phi(f_n,f^*_{\phi})\} 
   \ge 2^{i-2}\epsilon_n^2
    \end{eqnarray*}
We introduce the notation $M_{n,i}= 2^{i-2}\epsilon_n^2$ for a concise expression. Through the triangle inequality and (A4), we obtain the following variance bound
   \begin{equation}
   \label{eq:varbound}
   \begin{aligned}
  \sup_{f\in \cF_{n,i}}& \E\{\phi(Yf(\bX))-\phi(Yf_n(\bX))\}^2\\
   &\le C_2F_n^{2-\nu}\left(\sup_{f\in \cF_{n,i}} \cE_\phi(f, f^*_\phi)^{\nu}+ \cE_\phi(f_n, f^*_\phi)^{\nu}\right) \\
   &\le C_2 (1+4^\nu)F_n^{2-\nu}(2^{i-2}\epsilon_n^2)^\nu\\
   &=  C_2 (1+4^\nu) F_n^{2-\nu}M_{n,i}^\nu.
   \end{aligned}
   \end{equation}
Now, we have 
 \begin{equation}
 	\label{eq:peeling}
 \Pr\left\{\cE_\phi(\hat{f}_n, f^*_{\phi})\ge \epsilon_n^{2}\right\}\le \sum_{i=1}^{i_n^*}
 {\Pr}^*\left( \sup_{f\in \cF_{n,i}}\E_n(f)\ge M_{n,i}\right).
 \end{equation}
To bound the right-hand side, we apply Lemma \ref{lem:empineq} to the class of functions
	$$\cH_{n,i}=\{(\bx, y)\mapsto\phi(yf_n(\bx))-\phi(yf(\bx)):f\in \cF_{n,i}\},$$
with $\zeta=1/2$,  $F=D_1F_n$, $M=M_{n,i}$, and $v=v_{n,i}=D_2F_
n^{2-\nu}M_{n,i}^\nu$ where we let
	$$D_1=\frac{1}{8(2C_1)^{1-\nu}}D_2, \quad D_2=\max\{ C_2 (1+4^\nu), 64(2C_1)^{2-\nu}\}.$$
Note that for any $h\in\cH_{n,i}$, $\|h\|_\infty\le C_1\|f_n-f\|_\infty\le 2C_1F_n$, and $\sup_{h\in \cH_{n,i}}\textup{Var}(h(\bX, Y))\le C_2 (1+4^\nu) F_n^{2-\nu}M_{n,i}^\nu$ by (\ref{eq:varbound}). Since $D_1\ge 2C_1$ and $D_2\ge C_2(1+4^\nu)$, $\sup_{h\in \cH_{n,i}}\|h\|_\infty\le D_1F_n$, and $\sup_{h\in \cH_{n,i}}\textup{Var}(h(\bX, Y))\le v_{n,i}$. Now we will check (C1)-(C3) of Lemma \ref{lem:empineq}. Because $M_{n,i}\le 2C_1F_n$ for any $i\le i_n^*$ and $D_2\ge 64(2C_1)^{2-\nu} $, 
 	\begin{align*}
 	\frac{v}{F^2}=\frac{v_{n,i}}{D_1^2F_n^2}
 	&=\frac{D_2F_n^{2-\nu}(2C_1F_n)^\nu}{D_1^2F_n^2}\\
 	&\le \frac{D_2(2C_1)^{\nu}}{D_1^2}=\frac{64(2C_1)^{2-2\nu}(2C_1)^{\nu}}{D_2} 	\le 1
 	\end{align*}
and
 	\begin{align*}
 	M_{n,i}=M_{n,i}^{1-\nu}M_{n,i}^\nu &\le (2C_1F_n)^{1-\nu}M_{n,i}\\
 	& \le\frac{8(2C_1)^{1-\nu}D_1F_n^{2-\nu} M_{n,i}^\nu}{8D_1F_n}= \frac{v_{n,i}}{8D_1F_n}.
 	\end{align*}
Therefore, (C2) in Lemma \ref{lem:empineq} holds. For (C3), we first note that
    \begin{eqnarray*}
    H_B(\delta, \cH_{n,i}, \|\cdot\|_2) 
    \le H_B(C_1\delta, \cF_{n,i}, \|\cdot\|_2)
    \le H_B(C_1\delta, \cF_{n}, \|\cdot\|_2),
    \end{eqnarray*}
where the first inequality follows from (A1), and the second inequality follows from $\cF_{n,i}\subset\cF_{n}$.
Because $\int_{\zeta M_{n,i}/32}^{v_{n,i}^{1/2}}H_B(u, \cF_n, \|\cdot\|_2)^{1/2}\textup{d} u/M_{n,i}$ is non-increasing in $i$,
    \begin{equation}
    \label{eq:entcalc}
    \begin{aligned}
   & M_{n,i}^{-1}\int_{M_{n,i}/64}^{v_{n,i}^{1/2}}H_B^{1/2}(u, \cH_{n,i}, \|\cdot\|_2)\textup{d} u \\
    &\le  M_{n,1}^{-1}\int_{ M_{n,1}/64}^{v_{n,1}^{1/2}} H_B^{1/2}(C_1u, \cF_n, \|\cdot\|_2)\textup{d} u \\
    &\le  M_{n,1}^{-1}v_{n,1}^{1/2}H_B^{1/2}(C_1 M_{n,1}/64, \cF_n, \|\cdot\|_2) \\
    &\le (D_2F_n^{2-\nu})^{1/2}M_{n,1}^{\nu/2-1}  H_B^{1/2}(C_1\epsilon_n^2/128, \cF_n, \|\cdot\|_2) \\
    &\le C_3^{1/2}(D_2F_n^{2-\nu})^{1/2}(\epsilon_n/2)^{\nu-2}( n^{1/2}C_1^{1-\nu/2}2^{-7+7\nu/2}\epsilon_n^{2-\nu}F_n^{\nu/2-1})\\
    &\le (2^{-5+5\nu/2} C_3^{1/2} C_1^{1-\nu/2}D_2^{1/2})n^{1/2},
    \end{aligned}
    \end{equation}
where the fourth inequality is due to (A5). By taking $ C_3^{1/2}=2^{-13/2-5\nu/2}C_1^{\nu/2-1}D_2^{-1/2}$, 
(C3) of Lemma \ref{lem:empineq} is satisfied.
Furthermore, (\ref{eq:entcalc}) implies that
	\begin{align*}
	H_B(v^{1/2}_{n,i}, \cH_{n,i}, \|\cdot\|_2)^{1/2}
	&\le \frac{M_{n,i}}{v^{1/2}_{n,i}-M_{n,i}/64}M_{n,i}^{-1}\int_{M_{n,i}/64}^{v_{n,i}^{1/2}}H_B^{1/2}(u, \cH_{n,i}, \|\cdot\|_2)\textup{d} u \\
	&\le \frac{M_{n,i}}{v^{1/2}_{n,i}-M_{n,i}/64}M_{n,1}^{-1}\int_{M_{n,1}/64}^{v_{n,1}^{1/2}}H_B^{1/2}(C_1u, \cF_{n}, \|\cdot\|_2)\textup{d} u \\
	&\le \frac{M_{n,i}}{v^{1/2}_{n,i}-M_{n,i}/64}n^{1/2}2^{-23/2}\\
	&\le \frac{8}{7}\frac{M_{n,i}}{v^{1/2}_{n,i}}n^{1/2}2^{-23/2}=\frac{1}{7\times2^{17/2}}\frac{M_{n,i}}{v^{1/2}_{n,i}}n^{1/2}
	\end{align*}
where the last inequality is due to that $v^{1/2}_{n,i}\ge M_{n,i}/ 8$. On the other hand, since $v_{n,i}/(8D_1F_n)\ge M_{n,i}$,
	\begin{align*}
	n\frac{M_{n,i}^2}{16(4v_{n,i}+M_{n,i}D_1F_n/3)}
	\ge n\frac{M_{n,i}^2}{(64+2/3)v_{n,i}}
	\end{align*}
which is larger than $\frac{1}{7^2\times2^{17}}\frac{M_{n,i}^2}{v_{n,i}}n.$ Hence (C1) of Lemma  \ref{lem:empineq} is met.

 Applying Lemma \ref{lem:empineq} to each $\cH_{n,i}$, (\ref{eq:peeling}) is further bounded as 
    \begin{eqnarray*}
    \Pr\left\{\cE_\phi(\hat{f}_n, f^*_{\phi})\ge \epsilon_n^{2}\right\}
    &\le& \sum_{i=1}^{i_n^*}3\exp\left(-\frac{nM_{n,i}^2}{4(4v_{n,i}+M_{n,i}F_n/3)}\right) \\
    &\le& \sum_{i=1}^{\infty}3\exp(-C_4 nM_{n,i}^{2}/v_{n,i}\}) \\
     &\le& \sum_{i=1}^{\infty}3\exp(-C_5 (2^{i})^{2-\nu}n(\epsilon_n^2/F_n)^{2-\nu}) \\
     &\le& C_6\exp(-C_5 n(\epsilon_n^2/F_n)^{2-\nu})
    \end{eqnarray*}
for certain positive constants $C_4, C_5$, and $C_6$, which leads to the desired result. 
\end{proof}

\subsection{Generic convergence rate for the hinge loss}

We derive the convergence rate of the excess risk of the hinge loss under the conditions (A2), (A3), and (A5). Note that (A1) holds with $C_1=1$ for the hinge loss.
 We adopt the following lemma for the variance bound (A4).

\begin{lemma}[Lemma 6.1 of \cite{steinwart2007fast}]
\label{lem:hingevar}
Assume \textup{(N)} with the noise exponent $q\in[0,\infty]$. Assume $\|f\|_\infty\le F$ for any $f\in \cF$. For the hinge loss $\phi$, we have that, for any $f\in \cF$,
\bean
&&\E\left[\left(\phi(Yf(\bX))-\phi(Yf^*_\phi(\bX))\right)^2\right]\\
&&\le C_{\eta, q}(F+1)^{(q+2)/(q+1)}\left(\E\left[\phi(Yf(\bX))-\phi(Yf^*_\phi(\bX))\right]\right)^{q/q+1},
\eean
where $C_{\eta, q}=\left(\|(2\eta-1)^{-1}\|_{q,\infty}^q+1\right)\mathbf{1}(q>0)+1$ and $\|(2\eta-1)^{-1}\|_{q,\infty}^q$ is defined by 
$$\|(2\eta-1)^{-1}\|_{q,\infty}^{q}=\sup _{{t>0}}\left(t^q\Pr\left(\{\bX:|(2\eta(\bX)-1)^{-1}|>t\}\right)\right).$$
\end{lemma}

\begin{theorem}
\label{thm:hingecon}
Let $\phi$ be the hinge loss. Assume \textup{(N)} with the noise exponent  $q\in[0,\infty]$, and that \textup{(A2), (A3)}, and \textup{(A5)} are met. Let $\epsilon_n^2\asymp \max\{a_n, \delta_n\}$. Assume that  $n^{1-\iota}(\epsilon_n^2/F_n)^{(q+2)/(q+1)}\gtrsim 1$ for an arbitrarily small constant $\iota>0$. Then, the empirical $\phi$-risk minimizer $\hat{f}_{\phi,n}$ over $\cF_n$ satisfies 
\be
\label{eq:hinge-excess}
\E\left[\cE(\hat{f}_{\phi,n}, C^*)\right]\lesssim \epsilon_n^2,
\ee
where the expectation is taken over the training data.
\end{theorem}

\begin{proof}
By Zhang's inequality (Theorem 2.31 of \citep{steinwart2008support}), we have $\cE(\hat{f}_{\phi,n}, C^*)\le \cE_\phi(\hat{f}_{\phi,n}, f_\phi^*)$. Since (A4) is satisfied with $\nu=q/(q+1)$ by Lemma \ref{lem:hingevar}, Theorem \ref{thm:gen} implies that
	$$\Pr\left(\cE(\hat{f}_{\phi,n}, C^*)\ge\epsilon_n^2\right)\lesssim \exp(-Cn(\epsilon_n^2/F_n)^{(q+2)/(q+1)}),$$
for some universal constant $C>0$.
Since $\cE(\hat{f}_{\phi,n}, C^*)$ is bounded above by 1, the preceding display and the assumption $n^{1-\iota}(\epsilon_n^2/F_n)^{(q+2)/(q+1)}\gtrsim 1$ imply the desired result.
\end{proof}

\subsection{Entropy of the class of DNNs}

The following proposition states the upper bound of the $\delta$-entropy of a neural network function space.

\begin{proposition}[Lemma 3 of \cite{suzuki2018adaptivity}, Lemma 5 of \cite{schmidt2017nonparametric}]
\label{prop:entropy}
For any $\delta>0$,
\bean
&&\log \cN\left(\delta, \cF^{\textup{DNN}}(L,N,S, B, \infty), \|\cdot\|_\infty\right)\\
&&\le 2L(S+1)\log\left(\delta^{-1}(L+1)(N+1)(B\vee1)\right).
\eean
where $B\vee1=\max\{B, 1\}$.
\end{proposition}

\subsection{Proof of Theorem \ref{th:main_sb}}

The following proposition given by \cite{petersen2018optimal} proves that DNNs are good at approximating piecewise constant functions with smooth boundaries.

\begin{proposition}[Corollary 3.7 of \cite{petersen2018optimal}]
\label{thm:pieceapprox}
Let $d\ge2$, $\alpha, r>0, K\in\mbN$, and $T\in\mbN$. For any $C\in \cC^{\alpha, r, K, T}$ and any sufficiently small $\xi>0$, there exists a neural network
$$f(\cdot|\Theta)\in \cF^{\textup{DNN}}\left(L_0\log\left(1/\xi\right), N_0\xi^{-(d-1)/\alpha}, S_0\xi^{-(d-1)/\alpha}\log\left(1/\xi\right), B_0\xi^{-b_0}, 1\right),$$
where the positive constants $L_0, N_0, S_0, B_0$, and $b_0$ depend only on $d, \alpha, r$, and $K$, such that
$$\E\left|f(\bX|\Theta)-C(\bX)\right| \le\xi.$$
\end{proposition}

\begin{proof}[Proof of Theorem \ref{th:main_sb}]
We will check the conditions (A2), (A3), and (A5) in Section \ref{sec:gen}, and apply Theorem \ref{thm:hingecon} to complete the proof.
For (A2), let $\{\xi_n\}_{n\in\mbN}$ be a positive sequence such that $\xi_n\downarrow0$. Through Proposition \ref{thm:pieceapprox}, there exists $f_n\in\cF_n=\cF^{\textup{DNN}}(L_n, N_n, S_n, B_n, 1)$ such that  $\E|f_n(\bX)-C^*(\bX)| \le \xi_n$ with 
$L_n\lesssim \log(1/ \xi_n)$, $N_n\lesssim \xi_n^{-(d-1)/\alpha}$ and $S_n\lesssim  \xi_n^{-(d-1)/\alpha}\log(1/ \xi_n)$. Thus,
    \begin{eqnarray*}
        \cE_\phi(f_n, f_\phi^*)&=&  \E[\phi(Yf_n(\bX))-\phi(YC^*(\bX))] \\
        &\le& C_1\E|Yf_n(\bX)-YC^*(\bX)| \\
        &\le& C_1\E|f_n(\bX)-C^*(\bX)| \le C_1\xi_n,
    \end{eqnarray*}
 and hence (A2) and (A3) hold with $a_n=\xi_n$ and $F_n=1$.   
    
For (A5), let $\epsilon_n^2=C_1\xi_n$. Then, by Proposition \ref{prop:entropy}, 
    \begin{eqnarray*}
    &&\log \cN(\epsilon_n^2, \cF^{\textup{DNN}}(L_n, N_n, S_n, B_n, 1),\|\cdot\|_\infty) \\
    &&\le 2L_n(S_n+1)\log\left((\epsilon_n^2)^{-1}(L_n+1)(N_n+1)(B_n\vee1)\right) \\
    &&\lesssim \xi_n^{-(d-1)/\alpha}\log^2(\xi_n^{-1})\log(\epsilon_n^{-1}) \\
    &&\lesssim \epsilon_n^{-2(d-1)/\alpha}\log^3(\epsilon_n^{-1}).
    \end{eqnarray*}
In turn, (\ref{eq:entropy}) implies that
  (A5) is satisfied if we choose $\epsilon_n$ satisfying
$$(\epsilon_n^2)^{\frac{q+2}{q+1}+\frac{(d-1)}{\alpha}}\gtrsim n^{-1}\log^3(\epsilon_n^{-1}),$$
which leads to the best possible convergence rate
$$\epsilon_n^2 = 
\left(\frac{\log^3n}{n}\right)^{\frac{\alpha (q+1)}{\alpha (q+2)+(d-1)(q+1)}}$$
 and completes the proof by Theorem \ref{thm:hingecon}.
\end{proof}

\subsection{Proof of Theorem \ref{th:main_scp}}

We first introduce the smooth function approximation result of DNNs.

\begin{proposition}
\label{thm:smoothapprox}
For any function $f\in\cH^{\alpha, r}([0,1]^d)$ and any sufficiently small $\xi>0$, there exists a neural network 
$$f(\cdot|\Theta)\in \cF^{\textup{DNN}}\left(L_0\log\left(1/\xi\right), N_0\xi^{-d/\alpha}, S_0\xi^{-d/\alpha}\log\left(1/\xi\right), 1, F \right)$$
such that
$$\|f(\cdot|\Theta)-f \|_\infty\le \xi$$
where the constants $L_0,N_0,S_0$, and $F$  depend  only on $d,\alpha$ and $r$.
\end{proposition}

\begin{proof}
Theorem 5 of \cite{schmidt2017nonparametric} proves that 
for any $f\in\cH^{\alpha, r}([0,1]^d)$ and any integers $m\ge 1$ and $M\ge(\alpha+1)^d\vee (r+1)$, there exists 
 a neural network $f(\cdot|\Theta)\in \cF^{\text{DNN}}(L,N, S, 1,  \infty)$
 such that 
 $$\left\|f(\cdot|\Theta)-f(\cdot)\right\|_\infty\le 3^{d+1}2^{-m}(2r+1)M + 2^\alpha rM^{-\alpha/d},$$
where $L=8+(m+5)(1+ \lceil \log_2 d\rceil)$, $N= 12dM$, and $S=94d^2(\alpha+1)^{2d}M(m+6)(1+ \lceil \log_2 d\rceil)$.
By letting $M=( 2^{-(\alpha+1)}r^{-1}\xi^{-d/\alpha}$ and $m=\log_2\left((2r+1)6^{d+1}(2r)^{d/\alpha}\xi^{-d/\alpha-1}\right)$, we have
$L \lesssim L_0\log\left(1/\xi\right), N\lesssim N_0\xi^{-d/\alpha},
S\lesssim  S_0\xi^{-d/\alpha}\log\left(1/\xi\right)$, and
$\|f(\cdot|\Theta)-f \|_\infty\le \xi$. Finally, because
$\|f\|_\infty\le r$, we have $\|f(\cdot|\Theta)\|_\infty \le r+\epsilon$, and hence
we complete the proof with $F=r+\xi$.
\end{proof}

\begin{proof}[Proof of Theorem \ref{th:main_scp}]
For a given $\xi_n$, by Proposition \ref{thm:smoothapprox}, there exists $\tilde{\eta}_n$ such that  $\|\tilde{\eta}_n(\bx)-\eta(\bx)\|_\infty \le \xi_n$ with at most $C_1\log(1/ \xi_n)$ layers, $C_2\xi_n^{-d/\beta}$ nodes in each layer, and $C_3\xi_n^{-d/\beta}\log(1/ \xi_n)$ nonzero parameters for some positive constants $C_1, C_2$, and $C_3$. We construct the neural network $f_n$ by adding one layer to $\tilde{\eta}(\bx)$ to achieve
$$f_n(\bx)=2\left\{\sigma\left(\frac{1}{\xi_n}\left(\tilde{\eta}_n(\bx)-\frac{1}{2}\right)\right)
-\sigma\left(\frac{1}{\xi_n}\left(\tilde{\eta}_n(\bx)-\frac{1}{2}\right)-1\right)\right\}-1,
$$
where $\sigma$ denotes the ReLU activation function. 
Note that $f_n(\bx)$ is equal to $1$ if $\tilde{\eta}_n(\bx)\ge 1/2+\xi_n$, $(\tilde{\eta}_n(\bx)-1/2)/\xi_n$ if $1/2\le\tilde{\eta}_n(\bx)< 1/2+\xi_n$, and $-1$ otherwise. Let $A(4\xi_n)=\{\bx:|2\eta(\bx)-1|>4\xi_n\}$. Then, for $A(4\xi_n)$, $|f_n(\bx)-C^*(\bx)|=0$ because $\tilde{\eta}_n(\bx)-1/2=(\eta(\bx)-1/2)-(\tilde{\eta}_n(\bx)-\eta(\bx))\ge\xi_n$ when $2\eta(\bx)-1>4\xi_n$. Similarly, we can show that $\tilde{\eta}_n(\bx)-1/2<-\xi_n$ when $2\eta(\bx)-1<4\xi_n$. Therefore, by (N) we have 
   \begin{align*}
        \E[\phi(Yf_n(\bX))-\phi(YC^*(\bX))] 
       &= \int|f_n(\bx)-C^*(\bx)||2\eta(\bx)-1|\textup{d}P_X(\bx) \\
       &=  \int_{A(4\xi_n)^c}|f_n(\bx)-C^*(\bx)||2\eta(\bx)-1|\textup{d}P_X(\bx)\\
       &\le 8\xi_n\Pr(\{\bX:|2\eta(\bX)-1|\le 4\xi_n\})\asymp \xi_n^{q+1},
    \end{align*}
where the inequality in the last line holds since $\|f_n(\bx)\|_\infty\le1$.

Note that $f_n$ is also a DNN in which
the last layer of $f_n$ has a finite number of parameters, and the maximum of the parameters is bounded above by $\xi_n^{-1}$. Hence, we can construct the DNN class $\cF^{\textup{DNN}}(L_n, N_n, S_n, B_n, 1)$ containing $f_n$ with  
$L_n\lesssim \log(1/ \xi_n)$, $N_n\lesssim \xi_n^{-d/\beta}$, $S_n\lesssim  \xi_n^{-d/\beta}\log(1/ \xi_n)$, and $B_n\lesssim \xi_n^{-1}$.
Now, take $\epsilon_n^2\asymp \xi_n^{q+1}$ and observe that
 $$\log \cN(\epsilon_n^2, \cF^{\textup{DNN}}(L_n, N_n, S_n, B_n, 1),\|\cdot\|_\infty)\lesssim (\epsilon_n^2)^{-d/(q+1)\beta}\log^3(\epsilon_n^{-1})$$
through  Proposition \ref{prop:entropy}.
Since $H_B(\delta, \cF,\|\cdot\|_2)\le \log\cN(\delta/2, \cF,\|\cdot\|_\infty)$,
  (A5) is satisfied if we choose $\epsilon_n$ satisfying
$$ (\epsilon_n^2)^{\frac{q+2}{q+1}+\frac{d}{\beta}}\gtrsim n^{-1}\log^3(\epsilon_n^{-1}),$$
which leads to the best possible convergence rate
$$\epsilon_n^2 = 
\left(\frac{\log^3n}{n}\right)^{\frac{\beta (q+1)}{\beta (q+2)+d}}$$
 and completes the proof based on Theorem \ref{thm:hingecon}.
 \end{proof}

\subsection{Proof of Theorem \ref{th:main_mar}}

The main technique of the proof is to approximate a piecewise constant function using
a DNN with respect to the supremum norm on a specific subset of the domain, where this
subset depends on the function to be approximated. 

Let  $d\ge2$, $\alpha, r>0$, and $K\in\mbN$. Let $A_1,\dots, A_T\in\cA^{\alpha, r,K}$ 
be a disjoint with the form
$$A_t=\bigcap_{k=1}^K\left\{\bx\in[0,1]^d:x_{j_{(t,k)}}-g_{(t,k)}(\bx_{-j_{(t,k)}})\ge0\right\}.$$ 
Let $T\in\mbN$, and let
\be
\label{eq:cx}
C(\bx)=2\sum_{t=1}^T \mathbf{1}(\bx\in A_t)-1.
\ee
For a given $\xi>0$, define $B_{\xi}$ such that
\be
\label{eq:be}
B_{\xi}=\bigcap_{t=1}^T\left(A_t^c\cup\left\{\bx\in A_t:x_{j_{(t,k)}}-g_{(t,k)}(\bx_{-j_{(t,k)}})>\xi,\forall k\in[K]\right\}\right)
\ee
It turns out that any point in $B_{\xi}$ has the supremum norm from the
the decision boundary of $C(\bx)$ being larger than $\xi$.
The following theorem proves that a DNN approximates $C(\bx)$ well on $B_{\xi}$.

\begin{proposition}
\label{thm:approx}
Let $d\ge2$, $\alpha, r>0, K\in\mbN$, and $T\in\mbN$. For any $C\in \cC^{\alpha, r, K, T}$ and a sufficiently small $\xi>0$, there exists a neural network
$$f(\cdot|\Theta)\in \cF^{\textup{DNN}}\left(L_0\log\left(1/\xi\right), N_0\xi^{-(d-1)/\alpha}, S_0\xi^{-(d-1)/\alpha}\log\left(1/\xi\right), B_0\xi^{-b_0}, 1\right),$$
where the positive constants $L_0, N_0, S_0, B_0$, and $b_0$ depend  only on $d, \alpha, r, K$, and $T$, such that
$$\sup_{\bx \in B_{\xi}}
\left| f(\bx|\Theta)-C(\bx)\right|=0,$$
where $C(\bx)$ is the function defined in (\ref{eq:cx}), and
$B_{\xi}$ is defined in (\ref{eq:be}).
\end{proposition}

\begin{proof}
The proof is deferred to Section \ref{sec:approxproof}.
\end{proof}

\begin{proof}[Proof of Theorem \ref{th:main_mar}]
Let $\{\xi_n\}_{n\in\mbN}$ be a positive sequence such that $\xi_n\downarrow0$.  Based on Theorem \ref{thm:approx}, there exists $f_n\in\cF_n=\cF^{\textup{DNN}}(L_n, N_n, S_n, B_n, 1)$ such that  
$$\sup_{\bx \in B_{\xi_n}}
\left| f(\bx|\Theta)-C(\bx)\right|=0$$
with $L_n\lesssim \log(1/ \xi_n)$, $N_n\lesssim \xi_n^{-(d-1)/\alpha}$, $S_n\lesssim  \xi_n^{-(d-1)/\alpha}\log(1/ \xi_n)$, and $B_n\lesssim \xi_n^{-b_0}$ for some $b_0>0$. 

We now show that $B_{\xi_n}^c\subset B_{\xi_n}^*$, where $B_{\xi_n}^*=\{\bx:\text{dist}(\bx, D^*)\le\xi_n\}$. Suppose that $\bx\in B_{\xi_n}^c$. Then, there are $t\in[T]$ and $k\in[K]$ such that 
$|x_{j_{(t,k)}}-g_{(t,k)}(\bx_{-j_{(t,k)}})|\le\xi_n$. Let $\bx^*$ be the $d$-dimensional vector where the $j_{(t,k)}$-th component is equal to $g_{(t,k)}(\bx_{-j_{(t,k)}})$ and the other components are the same as the corresponding components of $\bx$, i.e., $x^*_{j_{(t,k)}}=g_{(t,k)}(\bx_{-j_{(t,k)}})$ and $\bx^*_{-j_{(t,k)}}=\bx_{-j_{(t,k)}}$. Clearly, $\bx^*$ is on the decision boundary $D^*$. Since
$\|\bx-\bx^*\|_2=|x_{j_{(t,k)}}-g_{(t,k)}(\bx_{-j_{(t,k)}})|\le\xi_n$, it follows that $\text{dist}(\bx, D^*)\le\xi_n$, which implies that $f_n(\bx)-C^*(\bx)=0$ for any $\bx\in (B_{\xi_n}^*)^c$ since $(B_{\xi_n}^*)^c\subset B_{\xi_n}$. Therefore, through the condition (M), 
  	\begin{align*}
        \E[\phi(Yf_n(\bX))-\phi(YC^*(\bX))] 
       &= \int|f_n(\bx)-C^*(\bx)||2\eta(\bx)-1|\textup{d}P_X(\bx) \\
       &= \int_{B_{\xi_n}^*}|f_n(\bx)-C^*(\bx)||2\eta(\bx)-1|\textup{d}P_X(\bx)\\
       &\le C\xi_n^{\gamma}.
    \end{align*}
for some constant $C>0$, and hence (A2) and (A3) hold with $a_n=C \xi_n^\gamma$
and $F=1$.

For (A5), if we take $\epsilon_n^2=C\xi_n^{\gamma}$, it follows that
 $$\log \cN(\epsilon_n^2, \cF^{\textup{DNN}}(L_n, N_n, S_n, B_n, 1),\|\cdot\|_\infty)\lesssim (\epsilon_n^2)^{-(d-1)/\alpha\gamma}\log^3(\epsilon_n^{-1}).$$
Since $H_B(\delta, \cF,\|\cdot\|_2)\le \log\cN(\delta/2, \cF,\|\cdot\|_\infty)$,
  (A5) is satisfied if we choose $\epsilon_n$ satisfying
$$(\epsilon_n^2)^{\frac{q+2}{q+1}+\frac{d-1}{\alpha\gamma}}\gtrsim n^{-1}\log^3(\epsilon_n^{-1}),$$
which leads to the best possible convergence rate
$$\epsilon_n^2 = 
\left(\frac{\log^3n}{n}\right)^{\frac{\alpha(q+1)}{\alpha (q+2)+(d-1)(q+1)/\gamma}},$$
 and completes the proof by Theorem \ref{thm:hingecon}.
\end{proof}

\subsection{Proof of Theorem \ref{th:main_cr}}

For the logistic loss, the following two lemmas are needed. The first lemma states that the $\phi$-risks of both the $\phi$-risk minimizer and the Bayes classifier are bounded. The second lemma provides the variance bound of the logistic loss.

\begin{lemma}
\label{lem:cebound}
Let $\phi$ be the logistic loss. Assume \textup{(E)} with $\lambda_n\asymp e^{-\tilde{F}_n}$. There then exist constants $C_1>0$ and $C_2>0$ such that
\bean
\cE_\phi(f_{\phi}^*)&\le& C_1\tilde{F}_ne^{-\tilde{F}_n} \\
\cE_\phi(\bar{f}_n^*)&\le& C_2\tilde{F}_ne^{-\tilde{F}_n}
\eean
where $\bar{f}_n^*=\tilde{F}_nC^*$.
\end{lemma}
\begin{proof}
Recall that $f_{\phi}^*(\bx)=\log(\eta(\bx)/(1-\eta(\bx)))$. We let
$$A_n=\{\bx:|f_{\phi}^*|> \tilde{F}_n\}$$.
It then follows that
    \begin{align*}
    \cE_\phi(f_{\phi}^*)
    &=\int \{\eta(\bx)\phi(f^*_\phi(\bx)) +  (1-\eta(\bx))\phi(-f^*_\phi(\bx))\} \textup{d}P_X(\bx) \nonumber\\
    &\le\phi(-\tilde{F}_n)\Pr(A_n)+
    \int_{A_n} \{\eta(\bx)\phi(f^*_\phi(\bx)) +  (1-\eta(\bx))\phi(-f^*_\phi(\bx))\} \textup{d}P_X(\bx) \nonumber\\
    &\le C_3\tilde{F}_n\lambda_n+
    \int_{A_n}\left[\frac{e^{f^*_\phi(\bx)}}{1+e^{f^*_\phi(\bx)}}\phi(f^*_\phi(\bx)) + \frac{1}{1+e^{f^*_\phi(\bx)}}\phi(-f^*_\phi(\bx)) \right]\textup{d}P_X(\bx). 
    \end{align*}
Let
$$G_n=\int_{A_n}\left[\frac{e^{f^*_\phi(\bx)}}{1+e^{f^*_\phi(\bx)}}\phi(f^*_\phi(\bx)) + \frac{1}{1+e^{f^*_\phi(\bx)}}\phi(-f^*_\phi(\bx)) \right]\textup{d}P_X(\bx). 
$$
We divide $A_n$ into two disjoint sets $\{\bx:f_{\phi}^*> \tilde{F}_n\}$ and $\{\bx:f_{\phi}^*<-\tilde{F}_n\}$. On $\{\bx:f_{\phi}^*>\tilde{F}_n\}$, we have \begin{eqnarray*}
 G_n&\le& \phi(\tilde{F}_n)+\sup_{z\ge \tilde{F}_n}\frac{\log(1+e^z)}{1+e^{z}} \\
    &\le&\log(1+e^{-\tilde{F}_n})+\frac{\log(1+e^{\tilde{F}_n})}{1+e^{\tilde{F}_n}} \\
    &\lesssim& \tilde{F}_ne^{-\tilde{F}}.
    \end{eqnarray*}
Similarly, we can show that $ G_n\lesssim \tilde{F}_ne^{-\tilde{F}_n}$ on $\{\bx:f_{\phi}^*<- \tilde{F}_n\}$, which 
implies $\cE_\phi(f_{\phi}^*)\lesssim \tilde{F}_n\lambda_n+\tilde{F}_ne^{-\tilde{F}_n}\asymp \tilde{F}_ne^{-\tilde{F}_n}$.

We use the similar argument above for $\bar{f}_n^*=\tilde{F}_nC^*(\bx)$. For $\{\bx:f_{\phi}^*> \tilde{F}_n\}$
    \begin{eqnarray*}
    \eta(\bx)\phi(\bar{f}_n(\bx)) +  (1-\eta(\bx))\phi(-\bar{f}_n(\bx))
    \le\log(1+e^{-\tilde{F}_n})+\frac{\log(1+e^{\tilde{F}_n})}{1+e^{\tilde{F}_n}} \lesssim \tilde{F}_ne^{-\tilde{F}_n},
    \end{eqnarray*}
and similarly we obtain the same upper bound on $\{\bx:f_{\phi}^*<- \tilde{F}_n\}$. 
\end{proof}

\begin{lemma}[Lemma 6.1. of \cite{park2009convergence}]
\label{lem:logisticvar}
Assume \textup{(N)} with the noise exponent $q\in[0,\infty]$. Assume $\|f\|_\infty\le F$ for any $f\in \cF$. Then, for the logistic loss $\phi$, we have that, for any $f\in \cF$,
 \bean
    \E\left[\left(\phi(f)-\phi(f^*_{\phi})\right)^2\right]
    \le C F\{\cE_\phi(f,f^*_{\phi})\}
  \eean
for some constant $C>0$.
\end{lemma}

\begin{proof}[Proof of Theorem \ref{th:main_cr}]
Let $\bar{f}_n^*=F_nC^*$. 
As in the proof of Theorem \ref{th:main_mar}, for a positive sequence $\{\xi_n\}_{n\in\mbN}$ approaching zero, we can find  $f_n\in\cF_n=\cF^{\textup{DNN}}(L_n, N_n, S_n, B_n, F_n)$ such that
 $$\log \cN(\epsilon_n^2, \cF_n,\|\cdot\|_\infty)\lesssim (\xi_n)^{-(d-1)/\alpha}\log^3(\xi_n^{-1})$$
and
$$\|f_n-\bar{f}_n^*\|_{\infty, B_{\xi_n}}=0,$$
where $B_{\xi_n}$ is defined in (\ref{eq:be}), with $C$ being the Bayes classifier.
Because $B_{\xi_n}^c\subset B_{\xi_n}^*=\{\bx:\text{dist}(\bx, D^*)\le\xi_n\}$, the condition (M) implies that
    \bean
       \int|f_n(\bx)-\bar{f}_n^*(\bx)|\textup{d}P_X(\bx) 
       = \int_{B_{\xi_n}^*}|f_n(\bx)-\bar{f}_n^*(\bx)|\textup{d}P_X(\bx)
       \le C_1F_n\xi_n^{\gamma},
    \eean
for some constant $C_1>0$. By Lemma \ref{lem:cebound} and the Lipschitz property of the logistic loss, we have
    \begin{eqnarray*}
    \cE_\phi(f_n,  f_\phi^*)
    &\le &  \cE_\phi(f_n, \bar{f}_n^*) + \cE_\phi( f_\phi^*) +  \cE_\phi( \bar{f}_n^*) \\
    &\le &  C_2\int|f_n(\bx)-\bar{f}_n^*(\bx)|\textup{d}P_X(\bx)+ C_3F_ne^{-F_n} \\
     &\le &  C_1C_2F_n\xi_n^{\gamma}+C_3F_ne^{-F_n}
     \end{eqnarray*}
for some positive constants $C_2$ and $C_3$. Recall that we have defined 
$$\kappa=\frac{\alpha\gamma}{\alpha\gamma+d-1}.$$
We now take $F_n=\kappa (\log n-3\log(\log n))$ and $\xi_n^\gamma=n^{-\kappa}\log^{3\kappa}n$ such that $\cE_\phi(f_n,  f_\phi^*)\lesssim F_ne^{-F_n} \asymp n^{-\kappa}\log^{3\kappa+1} n$,
and thus the conditions (A2) and (A3) in Section \ref{sec:gen} hold with $a_n=n^{-\kappa}\log^{3\kappa+1} n$ and $F_n= \kappa (\log n-3\log(\log n))$.

For (A5), let $\epsilon_n^2=n^{-\kappa}\log^{3\kappa+1}n$. Because $\kappa(d-1)/\alpha\gamma=(d-1)/(\alpha\gamma+d-1)=1-\kappa$, it follows that
    \begin{eqnarray*}
    \log \cN(\epsilon_n^2, \cF_n,\|\cdot\|_\infty)
    &\lesssim & \left(\frac{\log^3n}{n}\right)^{-\kappa(d-1)/\alpha\gamma}\log^{3} n \\\
    &\lesssim & n^{1-\kappa}\log^{3\kappa}n \\
    &\asymp & nF_n^{-1}\epsilon_n^2
    \end{eqnarray*}
which implies (A5) and completes the proof through Theorem \ref{thm:gen}
with Lemma \ref{lem:logisticvar}, which proves the condition (A4).
\end{proof}

\subsection{Proof of Proposition \ref{thm:approx}}
\label{sec:approxproof}

\newcommand{\rbr}[1]{\left( {#1} \right)}
\newcommand{\cbr}[1]{\left\{ {#1} \right\}}
\newcommand{\sbr}[1]{\left[ {#1} \right]}
\newcommand{\norm}[1]{\left\| {#1} \right\|}
\newcommand{\abs}[1]{\left| {#1} \right|}

\def\hstack{\textsf{hstack}}
\def\vstack{\textsf{vstack}}
\def\diag{\textsf{diag}}

Before we provide the proof of Theorem \ref{thm:approx}, we introduce some useful definitions and techniques for the construction of DNNs, which are mostly from \cite{petersen2018optimal}.

 For matrices $\bW_1,\dots,\bW_N$, we let $\diag\rbr{\bW_1,\dots, \bW_N}$ denote a block diagonal matrix whose diagonal matrices are $\bW_1,\dots, \bW_N$. When $\bW_i$ have the same number of rows, we let $\hstack\rbr{\bW_1,\dots, \bW_N}$ denote a concatenated matrix along the column, and when $\bW_i$ have the same number of columns, we let $\vstack\rbr{\bW_1,\dots, \bW_N}$ denote a concatenated matrix along the row.


For an index set $D\subset [d]$, a \textit{masking neural network} with $L$ layers,  denoted by $f(\cdot|\Theta_{D, L})$, where $\Theta_{D, L}=\rbr{(\bW^{(l)}, \mathbf{0})}_{l=1,\dots, L+1}$,
in which $\bW^{(1)}=\vstack\rbr{\bI(D), -\bI(D)}\in \mbR^{2d\times d}$ and $\bW^{(1)}=\bI_{2d}\in \mbR^{2d\times 2d}$ for $l=2,\dots, L$, and $\bW^{(L+1)}=\hstack\rbr{\bI(D), -\bI(D)}\in \mbR^{d\times 2d}$, where $\bI_d$ denotes a $d\times d$ identity matrix and $\bI(D)$ is a diagonal matrix where the $i$-th diagonal entry is equal to 1 if $i\in D$, and is zero otherwise. If $L=1$, we define $\Theta_{D, 1}= \rbr{(\bI(D), \mathbf{0})}$. The output of the masking neural network is equal to the masked input $\bI(D)\bx$, of which the $j$-th element is equal to $x_j$ if $j\in D$, and is zero otherwise. Note that 
$\norm{\Theta_{D, L}}_0=2d(L-2)+4|D|\le 2dL$.

Let $\Theta_1=\rbr{(\bW_{1}^{(l)}, \bfb_{1}^{(l)})}_{l=1,\dots, L_1+1}$ and $\Theta_2=\rbr{(\bW_{2}^{(l)}, \bfb_{2}^{(l)})}_{l=1,\dots, L_2+1}$ be two neural networks such that the input layer of $\Theta_1$ has the same dimension as the output layer of $\Theta_2$. Then, a \textit{stacked neural network} of $\Theta_1$ and $\Theta_2$ denoted by $f(\cdot|\Theta_1\bullet\Theta_2)$, where $\Theta_1\bullet\Theta_2$ is defined by
	\begin{align*}
	\Theta_1\bullet\Theta_2=\big(&(\bW_{2}^{(1)}, \bfb_{2}^{(1)},), \dots, (\bW_{2}^{(L_2)}, \bfb_{2}^{(L_2)}), \\
    &(\bW_{1}^{(1)}\bW_{2}^{(L_2+1)}, \bW_{1}^{(1)}\bfb_{2}^{(L_2+1)}+\bfb_{1}^{(1)}), \\
    & (\bW_{1}^{(2)}, \bfb_{1}^{(2)}),     \dots, (\bW_{1}^{(L_1+1)}, \bfb_{1}^{(L_1+1)})\big).
	\end{align*}
The stacked neural network $\Theta_1\bullet\Theta_2$ has $L_1+L_2$ layers and satisfies
$$f(\bx|\Theta_1\bullet\Theta_2)=f(f(\bx|\Theta_2)|\Theta_1)$$
for any input $\bx\in \mbR^d$. In addition, we have that $\norm{\Theta_1\bullet\Theta_2}_0\le\norm{\Theta_{1}}_0 + \norm{\Theta_{2}}_0+2C$, where $C$ is the constant equal to the multiplication of the input dimension of $\Theta_1$ and the output dimension of $\Theta_2$.

Let $\Theta_1=\rbr{(\bW_{1}^{(l)}, \bfb_{1}^{(l)})}_{l=1,\dots, L+1}$ and $\Theta_2=\rbr{\bW_{2}^{(l)}, \bfb_{2}^{(l)})}_{l=1,\dots, L+1}$ be two neural networks with the same number of layers and $d$-dimensional inputs. A \textit{concatenated neural network} of the two networks $\Theta_1$ and $\Theta_2$ denoted by $f(\cdot|\Theta_1\oplus\Theta_2)$, where $\Theta_1\oplus\Theta_2$ is defined by
$\Theta_1\oplus\Theta_2=\rbr{\rbr{\bW_{1,2}^{(l)}, \bfb_{1,2}^{(l)}}}_{l=1,\dots, L+1}$,
where $\bfb_{1,2}^{(l)}=\vstack\rbr{\bfb_{1}^{(l)}, \bfb_{2}^{(l)}}$ for $l=1,\dots, L+1$, $\bW_{1,2}^{(l)}=\diag\rbr{\bW_{1}^{(l)}, \bW_{2}^{(l)}}$ for $l=2,\dots, L+1$, and $\bW_{1,2}^{(1)}=\vstack\rbr{\bW_{1}^{(l1)}, \bW_{2}^{(l)}}$.
The concatenated neural network satisfies 
$$f(\bx|\Theta_1\oplus\Theta_2)=\vstack\rbr{f(\bx|\Theta_1), f(\bx|\Theta_2)}$$
for any input $\bx\in \mbR^d$, as well as $\norm{\Theta_1\oplus\Theta_2}_0=\norm{\Theta_{1}}_0 + \norm{\Theta_{2}}_0$.

We are ready to prove Proposition \ref{thm:approx}. We divide the proof into two steps. First we give the proof of approximation of the horizon functions, and then using the result, we prove Proposition \ref{thm:approx}.

\begin{lemma}[Approximation of horizon functions]
\label{lemma:hor}
Let  $d\ge2$, $\alpha, r>0$, and $K\in\mbN$. For a horizon function $\Psi_{g,j}$, where $g\in \cH^{\alpha, r}([0,1]^{d-1}), j\in [d]$, and $\xi>0$, define
$$B_{\xi, g, j}=\cbr{\bx\in[0,1]^d:x_j-g(\bx_{-j})>\xi}\cup\cbr{\bx\in[0,1]^d:x_j-g(\bx_{-j})<0}.$$
There then exists a neural network
$$f(\cdot|\Theta)\in \cF^{\textup{DNN}}\left(L_0\log\left(1/\xi\right), N_0\xi^{-(d-1)/\alpha}, S_0\xi^{-(d-1)/\alpha}\log\left(1/\xi\right), B_0\xi^{-b_0}, 1\right),$$
where the positive constants $L_0, N_0, S_0, B_0$, and $b_0$  depend  only on $d, \alpha$, and $r$, such that
$$\left\|f(\cdot|\Theta)-\Psi_{g,j}\right\|_{\infty, B_{\xi, g, j}} =0.$$
\end{lemma}

\begin{proof}
Without a loss of generality, assume $j=1$. By  Proposition \ref{thm:smoothapprox}, we can construct a neural network $\tilde{g}=f(\cdot|\Theta_g)$ on $[0,1]^{d-1}$ such that $\|f(\cdot|\Theta_g)-g\|_{\infty}<\xi/4$ with $|\Theta_{g}|\lesssim \log(1/\xi)$, $N_{\max}(\Theta_{g})\lesssim \xi^{-(d-1)/\alpha}$,  $\|\Theta_{g}\|_0\lesssim \xi^{-(d-1)/\alpha}\log(1/\xi)$, $\|\Theta_g\|_\infty\le 1$, and $\|f(\cdot|\Theta_g)\|_{\infty}\le r$. 
Define the map $\Phi:\mbR^d\to\mbR^d$ by 
$$\Phi(\bx)=(x_1-g(\bx_{-1})-\xi/4, \bx_{-1}).$$
Let $D_1=\{1\}$ and $D_{-1}=[d]\setminus \{1\}$. Let $\Theta_{\pm}=\rbr{(1,-1)^\top, (-\xi/4,0)^\top}$. Then, consider the network 
$$\Theta_{\Phi}=\rbr{\Theta_{\pm}\bullet \rbr{\Theta_{D_1, L} \oplus \Theta_{g}}} \oplus \Theta_{D_{-1}, L},$$
where $\Theta_{D_1, L}$ and $\Theta_{D_{-1}, L}$ are masking neural networks and $L=|\Theta_{g}|$. Clearly, we have 
$$f(\bx|\Theta_{\Phi})=(x_1-\tilde{g}(\bx_{-1})-\xi/4, \bx_{-1}),$$
with $\|\Theta_{\Phi}\|_0\lesssim \xi^{-(d-1)/\alpha}\log(1/\xi)$. 

Let $H(\bx)=\mathbf{1}(x_1\ge0)$. We now
construct a neural network that approximates $H(\bx)$.
Let $\Theta_{H}=\rbr{(\bW^{(1)}, \bfb^{(1)}), (\bW^{(2)}, \bfb^{(2)})}$, where $\bW^{(1)}=(W_{i,j}^{(1)})_{i=1,2; j=1,\ldots,d}$ with $W_{i,1}^{(1)}=\xi^{-1}$ and $W_{i,j}^{(1)}=0$ for $j=2,\dots,d$, and for $i=1,2$, $\bfb^{(1)}=(0, -1)^\top$, $\bW^{(2)}=(1, -1)$ and $\bfb^{(2)}=0$. It can be shown that
 $|H(\bx)-f(\bx|\Theta_{H})|\le \mathbf{1}(0\le\bx_1\le\xi/2)$ for every $\bx$ with $\|\Theta_H\|_\infty\lesssim \xi^{-1}$.
See Lemma A.2 of \cite{petersen2018optimal} for details.
 
Finally, define 
$$\Theta=\Theta_{H}\bullet\Theta_{\Phi}$$
so that $f(\bx|\Theta)=f(f(\bx|\Theta_\Phi)|\Theta_H)$. We then have 
    \begin{equations*}
	&\norm{\Psi_{g,1}-f(\bx|\Theta)}_{\infty, B_{\xi, g,j}}\\
    &\le \norm{H(\Phi)- H(f(\cdot|\Theta_\Phi))} _{\infty, B_{\xi, g,j}}
    + \norm{H(f(\cdot|\Theta_\Phi))-f(f(\cdot|\Theta_\Phi)|\Theta_H)}_{\infty,  B_{\xi, g,j}} .
	\end{equations*} 
We show that both terms on the right-hand side of the preceding display are zero.	

For the first term, we note that
    \bean
	&&\abs{H(\Phi(\bx))- H(f(\cdot|\Theta_\Phi))} \\
    &&=|\mathbf{1}(x_1- g(\bx_{-1})-\xi/4\ge0)-\mathbf{1}(x_1- \tilde{g}(\bx_{-1})-\xi/4\ge0)| \\
     &&\le\mathbf{1}\{(x_1- g(\bx_{-1})-\xi/4)(x_1- \tilde{g}(\bx_{-1})-\xi/4)\le0\}.
	\eean
Note also that, on $B_{\xi, g,j}$, $|x_1-g(\bx_{-1})-\xi/4|>\xi/4$, and by the construction of $\tilde{g}$, $\abs{g(\bx_{-1})-\tilde{g}(\bx_{-1})}<\xi/4$ for any $\bx\in[0,1]^d$. Combining these two facts, we obtain 
	\begin{equations*}
	&(x_1- g(\bx_{-1})-\xi/4)(x_1- \tilde{g}(\bx_{-1})-\xi/4)\\
    &= \frac{1}{2}\sbr{ \rbr{x_1-g(\bx_{-1})-\xi/4}^2 + \rbr{x_1-\tilde{g}(\bx_{-1})-\xi/4}^2-\rbr{g(\bx_{-1})-\tilde{g}(\bx_{-1})}^2 } \\
    &>\frac{1}{2}  \rbr{x_1-\tilde{g}(\bx_{-1})-\xi/4}^2 \ge   0
	\end{equations*}   
which implies that the first term is equal to zero.

For the second term, we have that, for every $\bx\in B_{\xi, g,j}$
    \begin{equations*}
	\abs{H(f(\cdot|\Theta_\Phi))-f(f(\cdot|\Theta_\Phi)|\Theta_H)}
     \le \mathbf{1}(0\le x_1-\tilde{g}(\bx_{-1})-\xi/4\le\xi/2) =0
	\end{equations*}
where the second ineqaulity holds since $x_1-\tilde{g}(\bx_{-1})-\xi/4=x_1-g(\bx_{-1})+g(\bx_{-1})-\tilde{g}(\bx_{-1})-\xi/4<0$ if $x_1-g(\bx_{-1})<0$ and $x_1-g(\bx_{-1})+g(\bx_{-1})-\tilde{g}(\bx_{-1})-\xi/4>\xi/2$ if $x_1-g(\bx_{-1})>\xi$. Thus, $ \norm{H(\Phi)- H(f(\cdot|\Theta_\Phi))} _{\infty, B_{\xi, g,j}}=0$, which completes the proof.
\end{proof}

\begin{proof}[Proof of Proposition \ref{thm:approx}]
We give the proof only for the case of $T=1$. An extension of the cases $T\ge2$ is straightforward. Thus, we omit the subscript $t$ in all expressions. 

Let $f(\cdot|\Theta_k)$ be a neural network such that
$$|f(\bx|\Theta_k)-\mathbf{1}(x_{j_{k}}-g_{k}(\bx_{-j_{k}})\ge0)|=0$$
for any $\bx\in B_{\xi, k}=\cbr{\bx\in[0,1]^d:x_{j_k}-g_k(\bx_{-j_k})>\xi}\cup\cbr{\bx\in[0,1]^d:x_{j_k}-g_k(\bx_{-j_k})<0}$, as in Lemma \ref{lemma:hor}. 
Define the neural network $f(\mathbf{z}|\Theta_+)$ with $K$-dimensional inputs as
$$f(\mathbf{z}|\Theta_+)=2\sigma\rbr{\sum_{k=1}^Kz_k-(K-1)}-1,$$ where
 $\sigma$ denotes the ReLU activation function, and define 
$$\Theta=\Theta_+\bullet(\Theta_1 \oplus \cdots \oplus \Theta_K).$$
We now show that
$$\left\|\left\{2\mathbf{1}(\cdot\in A_t)-1\right\}-f(\cdot|\Theta)\right\|_{\infty, B_{\xi}}=0.$$
If $\bx\in A_t^c$, then there is $k^*$ such that $f(\bx|\Theta_{k^*})=0$. Hence, $\sum_{k=1}^K f(\bx|\Theta_k)\le K-1$ and thus $f(\bx|\Theta)=-1$. If  $\bx\in \left\{\bx\in A_t:x_{j_{k}}-g_{k}(\bx_{-j_{k}})>\xi,\forall k\in[K]\right\}$, then $f(\bx|\Theta_k)=1$ for all $k$, and hence $f(\bx|\Theta)=1$.
\end{proof}

\subsection{DNN architectures used for the experiments}
\label{sec:cnn}

For the MNIST dataset, we used a DNN with five hidden layers, whose numbers of nodes were 1200, 600, 300, 150, and 150, respectively. All hidden layers are followed by batch normalization \citep{ioffe2015batch}. In addition, for the SVHN and CIFAR10 datasets, we used the CNN models whose architectures are provided in Table \ref{tab:cnn}.

\begin{table}
\centering
\begin{tabular}{c|c}
\hline
SVHN&CIFAR10\\
\hline
\multicolumn{2}{c}{$32\times 32$ RGB images}\\
\hline
$3\times 3$ conv. 64 ReLU&$3\times 3$ conv. 96 ReLU\\
$3\times 3$ conv. 64 ReLU&$3\times 3$ conv. 96 ReLU\\
$3\times 3$ conv. 64 ReLU&$3\times 3$ conv. 96 ReLU\\
\hline
\multicolumn{2}{c}{$2\times 2$ max-pool, stride 2}\\
\multicolumn{2}{c}{dropout, $p=0.5$}\\
\hline
$3\times 3$ conv. 128 ReLU&$3\times 3$ conv. 192 ReLU\\
$3\times 3$ conv. 128 ReLU&$3\times 3$ conv. 192 ReLU\\
$3\times 3$ conv. 128 ReLU&$3\times 3$ conv. 192 ReLU\\
\hline
\multicolumn{2}{c}{$2\times 2$ max-pool, stride 2}\\
\multicolumn{2}{c}{dropout, $p=0.5$}\\
\hline
$3\times 3$ conv. 128 ReLU&$3\times 3$ conv. 192 ReLU\\
$1\times 1$ conv. 128 ReLU&$1\times 1$ conv. 192 ReLU\\
$1\times 1$ conv. 128 ReLU&$1\times 1$ conv. 192 ReLU\\
\hline
\multicolumn{2}{c}{global average pool, $6\times 6\to 1\times 1$}\\
\hline
FC $128 \to 1$& FC $192 \to 1$\\
\hline
\end{tabular}
\caption{CNN models used in our experiments over SVHN and CIFAR-10. All convolutional (conv.) and   fully connected (FC) layers are followed by the batch normalization.}
\label{tab:cnn}
\end{table}

\bibliographystyle{plainnat}
\bibliography{reference-deep}

\end{document}